\documentclass{article}

     \PassOptionsToPackage{compress}{natbib}

\usepackage[preprint]{neurips_2022}

\usepackage{natbib}
\usepackage[utf8]{inputenc}
\usepackage[T1]{fontenc}
\usepackage{hyperref}
\usepackage{url}
\usepackage{booktabs}
\usepackage{array}
\usepackage{amsfonts}
\usepackage{nicefrac}
\usepackage{microtype}
\usepackage{xcolor}

\hypersetup{
  colorlinks = true,
  urlcolor   = [rgb]{0,0,1},
  linkcolor  = [rgb]{0,0,0.5},
  citecolor  = [rgb]{0,0,0.5}
}

\usepackage{amsthm}
\usepackage{amsmath}
\usepackage{amssymb}
\usepackage{mathrsfs}
\usepackage{graphicx}

\usepackage{zref-clever}

\let\cref\zcref
\let\Cref\zcref

\zcsetup{
  cap        = true,   
  abbrev     = false   
}

\AddToHook{env/theorem/begin}{\zcsetup{countertype={theorem=theorem}}}
\AddToHook{env/proposition/begin}{\zcsetup{countertype={theorem=proposition}}}
\AddToHook{env/lemma/begin}{\zcsetup{countertype={theorem=lemma}}}
\AddToHook{env/corollary/begin}{\zcsetup{countertype={theorem=corollary}}}
\AddToHook{env/question/begin}{\zcsetup{countertype={theorem=remark}}} 
\AddToHook{env/hyp/begin}{\zcsetup{countertype={theorem=hypothesis}}}
\AddToHook{env/definition/begin}{\zcsetup{countertype={theorem=definition}}}
\AddToHook{env/remark/begin}{\zcsetup{countertype={theorem=remark}}}
\AddToHook{env/example/begin}{\zcsetup{countertype={theorem=example}}}
\AddToHook{env/concl/begin}{\zcsetup{countertype={theorem=result}}}      
\AddToHook{env/note/begin}{\zcsetup{countertype={theorem=note}}}
\AddToHook{env/conj/begin}{\zcsetup{countertype={theorem=proposition}}}  
\AddToHook{env/assumption/begin}{\zcsetup{countertype={theorem=theorem}}} 

\usepackage{comment}
\usepackage{algorithm}
\usepackage{algpseudocode}

\theoremstyle{plain}
\newtheorem{theorem}{Theorem}[section]
\newtheorem{proposition}[theorem]{Proposition}
\newtheorem{lemma}[theorem]{Lemma}
\newtheorem{corollary}[theorem]{Corollary}

\theoremstyle{definition}
\newtheorem{definition}[theorem]{Definition}

\newtheorem{example}[theorem]{Example}

\numberwithin{equation}{section}

\renewcommand{\paragraph}[1]{\noindent{\bf #1}}
\usepackage{wrapfig}
\usepackage{caption}
\usepackage{subcaption}
\usepackage{placeins}

\usepackage{bm} 

\newcommand{\R}{\mathbb{R}}

\DeclareMathOperator{\rank}{\mathrm{rank}}

\newcommand{\qvec}{q}
\newcommand{\kvec}{k}
\newcommand{\dimqkn}{d_{\mathrm{qk},n}}

\newcommand{\neval}{n_{\mathrm{eval}}}
\newcommand{\ntrain}{n_{\mathrm{train}}}

\author{%
  Tomohiro Hayase\\
  AIST
  \And
  Ryo Karakida\\
  AIST
}

\begin{document}
\title{A Unified Framework for Critical Scaling \\ of Inverse Temperature in Self-Attention}
\maketitle

\newcommand{\avg}{\mu}
\newcommand{\stnorm}{\mathcal{N}(0,1)}
\newcommand{\numsample}{n_s}

\newcommand{\invavg}{D^{-1}}
\newcommand{\trunc}{K}

\newcommand{\sqdA}{\sqrt{d}A}
\newcommand{\Ylln}{Y}
\newcommand{\Ycntr}{Y^f}
\newcommand{\Ypoly}{Y^Q}
\newcommand{\Ypolylin}{Y_\mathrm{lin}^Q}

\newcommand{\Ylin}{Y^f_\mathrm{lin}}

\newcommand{\smat}{S}
\newcommand{\ptn}{Z}
\newcommand{\ntoken}{\ell}
\newcommand{\dimqk}{d_\mathrm{qk}}
\renewcommand{\Pr}{\mathbb{P}}

\begin{abstract}
Length-dependent logit rescaling is widely used to stabilize long-context self-attention,
but existing analyses and methods suggest conflicting inverse-temperature laws for the
context length $n$, ranging from $(\log n)^{1/2}$ to $\log n$ and $(\log n)^2$.
We provide a general theory showing that the desirable scale is determined by the
gap-counting function $N_n$ of each attention row. Counting how many competitors
lie within each gap from the maximum, we define an upper-tail accumulation scale
and prove that it gives the critical inverse-temperature scale for softmax
concentration: below this scale, the top competitors remain unseparated,
whereas above it, the attention entropy collapses. This framework unifies prior
scaling laws as different $N_n$ and yields a direct diagnostic for attention-score
families, from idealized theoretical models to more practical transformers.
\end{abstract}

\section{Introduction}

Self-attention \citep{vaswani2017attention} is a central component of
modern deep learning systems. 
To understand the mechanism of self-attention, theoretical studies have been actively developed. 
One standard approaches in theoretical analysis is to take large-degree-of-freedom limits, such as the limits of large embedding dimension and large number of heads \citep{hron2020infinite,bordelon2024infinite,sakai2025inf}. In particular, motivated by the practical importance of inference with long contexts, recent work has begun to explore the infinite-context-length limit. For example, long contexts can induce token clustering as metastable states~\citep{geshkovski2024mathematical,bruno2025multiscale},
as well as rank collapse \citep{saada2025mind}. Moreover, increasing the inverse temperature with the context length changes the phase from rank collapse to entropy collapse, with a critical scaling separating the two regimes \citep{giorlandino2025two,chen2025critical}.

Despite this progress, our understanding of the desirable scaling for long contexts is limited.
The  prior theory has identified critical inverse-temperature scalings only in restricted
stochastic and deterministic settings, such as Gaussian-logit (i.e., random-energy-style) analyses and equicorrelated self-attention models \citep{giorlandino2025two,chen2025critical}. In practice, inverse temperatures that increase with the context length are also widely used to avoid the instability of inference for long contexts \citep{bai2023qwen,nakanishi2025scalable,peng2023yarn}. As reviewed in
\cref{subsec:related-work}, these mechanisms can be placed on a common
logarithmic-exponent scale, but the resulting exponents are not
consistent across models and arguments. This suggests that the exponent
is not determined by context length alone. What is missing is a
principle that identifies, from the attention scores themselves, the
scale at which softmax changes regime.

This paper develops such a principle through a deterministic
gap-counting framework. Instead of organizing an attention row by
ordinal rank, we count how many competitors lie within each score gap
from the maximum. The exponential growth rate of this counting curve
defines a row-wise critical scale, the upper-tail accumulation scale
(\cref{def:upper-crowding}). Our main theorem shows that this scale
controls the inverse-temperature scale separating the diffuse and
concentrated softmax regimes for arbitrary deterministic score
families, without assuming Gaussian logits, independent scores, random
initialization, or a special token construction.

The main contribution is a unified gap-counting framework: the
cumulative function $N_n(t)$ explains both idealized theoretical
models and actual transformers, and yields a calibration principle
for the inverse-temperature exponent. The upper-tail accumulation
scale $\Lambda_n$ keeps softmax away from both non-separation of the
top competitors and entropy collapse, so any non-degenerate
logarithmic schedule has $\xi_\beta=\xi_\Lambda$
(\cref{cor:exponent-selection}); existing temperature rules become
different $N_n$ and hence different $\xi_\Lambda$, with the legacy
$\beta_n\propto\log n$ as the $\xi_\Lambda=1$ slice. In two
complementary settings (inference-time $\beta$-sweep at extrapolation
lengths and training-time learning of an inverse-temperature
vector), $\hat\xi_\Lambda$ read directly from attention scores aligns
with the empirical optimum $\xi_\beta^*$.
\Cref{sec:xi-sweep} presents the motivating sweeps, and
\cref{sec:theory} develops the deterministic theory and applies its
empirical readings to concrete self-attention scores.

\newcommand{\PP}{\mathbb P}
\newcommand{\EE}{\mathbb E}
\newcommand{\one}{\mathbf 1}

\section{Preliminaries}
\label{sec:preliminaries}

\subsection{Notation}
\label{sec:notation}

Let the context length be $n$.
For $n\ge 2$, let $z_{n,1},\dots,z_{n,n}\in\mathbb R$. The softmax
probability of token $j$ at inverse temperature $\beta>0$ is
$p_{n,j}(\beta):=e^{\beta z_{n,j}}/\sum_{\ell=1}^n e^{\beta z_{n,\ell}}$.
The concrete instance we have in mind is a fixed-query self-attention row of
the standard scaled dot-product form \citep{vaswani2017attention},
\begin{equation}\label{eq:sdp-attention}
 z_{n,j}
 =\frac{\langle W_n^Q x_{n,i_n},\,W_n^K x_{n,j}\rangle}{\sqrt{\dimqkn}}
 =\frac{\langle \qvec_{n,i_n},\kvec_{n,j}\rangle}{\sqrt{\dimqkn}},
\end{equation}
with token embeddings $x_{n,j}\in\mathbb R^{d_n}$, weight matrices
$W_n^Q,W_n^K\in\mathbb R^{\dimqkn\times d_n}$, query/key vectors
$\qvec_{n,i}:=W_n^Q x_{n,i}$, $\kvec_{n,j}:=W_n^K x_{n,j}$, and a fixed
query index $i_n$. The theory below uses only the resulting score sequence
$(z_{n,j})_j$, so the embedding-weight factorization is absorbed into
$\qvec,\kvec$ from \cref{sec:theory} onward. For positive sequences, $f_n\sim g_n$ means
$f_n/g_n\to 1$ and $f_n\asymp g_n$ means $c_1 g_n\le f_n\le c_2 g_n$
for some constants $0<c_1\le c_2$ and all sufficiently large $n$. We
study the asymptotic regime $n\to\infty$ with $\beta=\beta_n$, and
parametrise scaling laws by the exponent $\xi_\beta\ge 0$ in
$\beta_n\asymp(\log n)^{\xi_\beta}$.

\subsection{Related work}
\label{subsec:related-work}
\label{subsec:prior-comparison}

\begin{wraptable}{r}{0.46\linewidth}
\vspace{-\baselineskip}
\centering
\small
\setlength{\tabcolsep}{4pt}
\renewcommand{\arraystretch}{1.05}
\begin{tabular}{@{}c >{\raggedright\arraybackslash}p{0.34\textwidth}@{}}
\toprule
$\xi_\beta$ & Mechanism (prior work) \\
\midrule
$1/2$ & Gaussian-logit analyses \citep{giorlandino2025two,anson2025scaleinvariant} \\
$1$   & SSMax \citep{nakanishi2025scalable}, Qwen LogN \citep{bai2023qwen}; equicorrelated phase transition \citep{chen2025critical} \\
$2$   & YaRN  \citep{peng2023yarn} \\
\bottomrule
\end{tabular}
\caption{Prior work grouped by the implicit exponent $\xi_\beta$ in $\beta_n\asymp(\log n)^{\xi_\beta}$.}
\label{tab:priorwork}
\end{wraptable}

Prior choices for the softmax inverse temperature in self-attention each
fix $\xi_\beta$ a priori (\cref{tab:priorwork}).
\emph{Gaussian-logit analyses ($\xi_\beta=1/2$).} Drawing an analogy with the
Random Energy Model \citep{derrida1981random}, \citet{giorlandino2025two}
conjecture $\beta_n\asymp\sqrt{\log n}$ as the scaling that balances rank
and entropy collapse. Under a marginally Gaussian attention-logit
assumption, \citet{anson2025scaleinvariant} derive a position-dependent
multiplicative factor $a_t\asymp\sqrt{\log t}$ that yields scale-invariant
total attention.
\emph{Explicit logarithmic temperature ($\xi_\beta=1$).} SSMax
\citep{nakanishi2025scalable} multiplies attention logits by $s\log n$
with $s$ learned per head, and Qwen \citep{bai2023qwen} adopts a
length-dependent LogN multiplier of order $\log n/\log \ntrain$ outside the
training window $\ntrain$. \citet{chen2025critical} prove a phase transition
at $\beta_n\asymp\log n$ in a deterministic equicorrelated (``simplex'')
configuration.
\emph{Doubled rescaling ($\xi_\beta=2$).} YaRN \citep{peng2023yarn} multiplies
both the query and key sides of the rotary embedding by a common
$\log n$-dependent multiplier, doubling the exponent
(\cref{sec:yarn-derivation}).
The values $\xi_\beta\in\{1/2,\,1,\,2\}$ are the only ones with widely cited
mechanisms; intermediate values appear naturally once $\xi_\beta$ is treated as
a continuous parameter, as in our $\xi_\beta$-sweep (\cref{sec:xi-sweep}) and
our model-free gap-counting theory (\cref{sec:theory}).

\section{A minimal empirical check: $\xi_\beta=1$ is not universal}
\label{sec:xi-sweep}

Prior mechanisms in \cref{tab:priorwork} fix the logarithmic
inverse-temperature exponent $\xi_\beta$ a priori. We use the
single-parameter ansatz $\beta_n=c(\log n)^{\xi_\beta}$ only as a probe:
the goal is to check whether the logarithmic choice $\xi_\beta=1$ is
empirically forced.

\cref{fig:body-four-panel} gives two complementary checks.
Three context-length quantities appear repeatedly below: the per-token
context length $n$ of \cref{sec:notation}, the \emph{evaluation context
length} $\neval$ (the inference sliding-window size, fixed per
evaluation, with $n\le\neval$), and the \emph{training context length}
$\ntrain$ (the model's training window).
\label{subsubsec:sweep-qwen}%
First, we evaluate the pretrained Qwen-7B-Chat at extrapolation
lengths $\neval>\ntrain$ ($\ntrain=8192$ for Qwen-7B-Chat),
sweeping the inverse-temperature exponent $\xi_\beta$ in the
length-dependent multiplier $\beta_n=c(\log n)^{\xi_\beta}$, applied per
token, while dynamic-NTK RoPE (\cref{para:dynamic-ntk-rope}) is held
fixed (the RoPE base is rescaled per $\neval$). The perplexity sweep, scored on the back half of
each sliding window so that every scored token is conditioned on at
least $n/2$ preceding tokens (formal definition in
\cref{subsec:setup-qwen}), is reported per $\neval$. On both PG-19
(PG19) \citep{rae2020compressive} and Proof-Pile-2 (PP2)
\citep{azerbayev2024llemma}, $\xi_\beta=1$ is not a local minimum: the
perplexity minima on the tested grid occur around
$\xi_\beta^*\in[2,2.5]$, near the $\xi_\beta=2$ regime of YaRN
\citep{peng2023yarn}. Qwen's default $\log n$ scaling is therefore
not optimal on the tested datasets and extrapolation lengths.

\label{subsubsec:sweep-nanogpt}%
Second, the SSMax design of \citet{nakanishi2025scalable} was motivated
by a probe experiment in which a learnable per-context-length
inverse-temperature vector trained from scratch was fit well by
$a_1\log n + a_2$. We replay this probe on GPT-124M
\citep{radford2019language} trained from scratch with the nanoGPT
codebase \citep{karpathy2022nanogpt}, with a length-$\neval=1024$
learnable vector $\beta_n$ (one parameter per context position; protocol
in \cref{subsec:nanogpt-pn-trajectory}) but
generalize the fit to $\beta_n=a_1(\log n)^{\xi_\beta}$ and sweep $\xi_\beta$. On both
SlimPajama (SP) \citep{soboleva2023slimpajama} and
OpenWebText (OWT) \citep{gokaslan2019openwebtext}, across the reported training snapshots, the
least-squares MSE for $n\ge 64$ has its minimum at
$\xi_\beta^*\in\{0.40, 0.50\}$ across SlimPajama and OpenWebText; the
SSMax choice $\xi_\beta=1$ is therefore not optimal even on its own
motivating experiment.

\begin{figure}[t]
\centering
\captionsetup[subfigure]{labelformat=parens,labelfont=small,textfont=small,
                         skip=1pt,position=bottom,aboveskip=0pt,belowskip=0pt}
\begin{subfigure}[t]{0.245\linewidth}\centering
  \includegraphics[width=\linewidth]{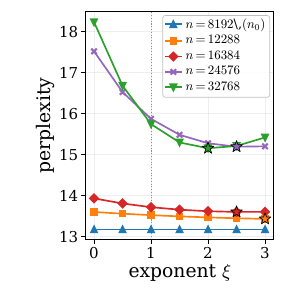}
  \caption{PG-19}\label{fig:ppl-xi-7b-pg19}
\end{subfigure}\hfill
\begin{subfigure}[t]{0.245\linewidth}\centering
  \includegraphics[width=\linewidth]{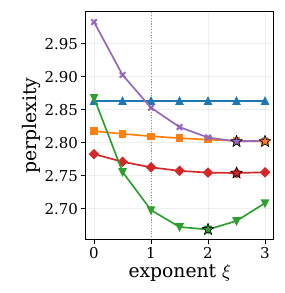}
  \caption{Proof-Pile-2}\label{fig:ppl-xi-7b-pp}
\end{subfigure}\hfill
\begin{subfigure}[t]{0.245\linewidth}\centering
  \includegraphics[width=\linewidth]{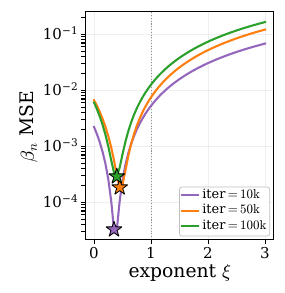}
  \caption{SlimPajama}\label{fig:nanogpt-pn-sp}
\end{subfigure}\hfill
\begin{subfigure}[t]{0.245\linewidth}\centering
  \includegraphics[width=\linewidth]{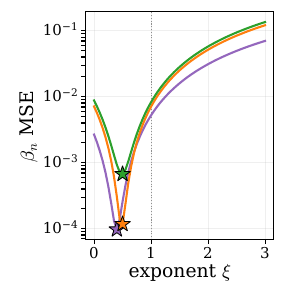}
  \caption{OpenWebText}\label{fig:nanogpt-pn-owt}
\end{subfigure}
\vspace{-4pt}
\caption{Sweeping $\xi_\beta$ shows that optima need not occur at $\xi_\beta=1$. (a, b)
Inference-time perplexity (back-half sliding window; \cref{subsec:setup-qwen}) is minimized
around $\xi_\beta^*\in[2,2.5]$.
(c, d) The bias-free least-squares fit of the learned $\beta_n$ to
$a_1(\log n)^{\xi_\beta}$ ($n\ge 64$) has its MSE minimum below $\xi_\beta=1$ on
both datasets. The dotted line marks $\xi_\beta=1$.
Doc-level SE on PPL is at most $10\%$ of the cell mean in (a, b);
residual-bootstrap half-IQR on MSE is at most $8\%$ of the
point-estimate MSE in (c, d).}
\label{fig:body-four-panel}
\end{figure}

Both sweeps therefore give concrete cases in which $\xi_\beta\neq 1$ is
preferred (full protocols, grids, companion sweeps, and per-context
numerics in \cref{sec:xi-sweep-numerics}). They select different
values of $\xi_\beta$, because perplexity and the learned $\beta_n$ are
different spatially-averaged observables of the same representation.
\Cref{sec:theory} characterizes the general cause of softmax collapse:
$\xi_\beta$ is fixed by the row-wise $N_n$ of attention scores.

\section{Deterministic theory}
\label{sec:theory}

The deterministic theory of softmax collapse is most naturally phrased in
terms of how many competitors lie within a given gap of the winner, not in
terms of which ordinal rank a competitor occupies: a one-step rank change
can correspond to a negligible score change on one row and a large one on
another, whereas increasing the gap by $\delta$ always suppresses a
competitor's weight by the same factor $e^{-\beta\delta}$. Grouping the softmax
partition function by gap rewrites it as a transform of a cumulative
gap-counting curve, and the inverse temperature at which softmax
collapses is then read off the exponential growth rate of that curve, namely the
\emph{upper-tail accumulation scale} $\Lambda_n$ of \cref{subsec:contact}.
\Cref{subsec:collapse} first fixes the collapse observables against
which $\Lambda_n$ is identified as the relevant scale.

\subsection{Collapse observables and the critical scale}
\label{subsec:collapse}

\begin{definition}[Collapse observables]\label{def:collapse}
Adopt the softmax probabilities $p_{n,j}(\beta)$ from \cref{sec:notation}.
The \emph{Shannon entropy} is $H_n(\beta):=-\sum_{j=1}^n p_{n,j}(\beta)\log p_{n,j}(\beta)$,
and the \emph{top-two weight gap} $D_n(\beta)$ is the difference
between the largest and second-largest entries of
$(p_{n,j}(\beta))_{j=1}^n$. Along a positive deterministic sequence
$(\beta_n)$, we say \emph{top-two collapse} holds if
$D_n(\beta_n)\to 0$, \emph{entropy collapse} if $H_n(\beta_n)\to 0$,
and \emph{rank collapse} if
$\max_{i,j}\bigl|p_{n,i}(\beta_n)-p_{n,j}(\beta_n)\bigr|\to 0$.
\end{definition}

\begin{definition}[Critical scale]\label{def:coarse-critical-scale}
A positive deterministic sequence $(X_n)$ is a
\emph{rank--entropy scale} (\emph{RE scale}) if, for every positive
deterministic sequence $(\beta_n)$,
\[
 \beta_n/X_n\to 0 \;\Longrightarrow\; \text{rank collapse},
 \qquad
 \beta_n/X_n\to\infty \;\Longrightarrow\; \text{entropy collapse},
\]
in the sense of \cref{def:collapse}, and a
\emph{top-two--entropy scale} (\emph{TE scale}) if the subcritical
conclusion is replaced by top-two collapse. Each is unique up to
$\asymp$ when it exists.
\end{definition}

We refer to TE and RE scales collectively as \emph{critical scales}:
each is a deterministic sequence that separates a subcritical regime,
in which a collapse observable vanishes, from a supercritical regime in
which entropy collapses. Rank collapse is the standard
transformer-literature notion of convergence of an attention row toward
the uniform law $(1/n,\dots,1/n)$
\citep{dong2021attention,noci2022signal,geshkovski2024mathematical},
strictly stronger than top-two collapse.

The primary scale here is $\Lambda_n$ of \cref{subsec:contact}, which
\cref{thm:deterministic-} identifies as a critical scale
(top-two/entropy sense of \cref{def:coarse-critical-scale}) and which
inverse-temperature calibration requires
(\cref{cor:exponent-selection}). \Cref{cor:rank-boundary-diffuse}
additionally upgrades $\Lambda_n$ to the standard rank-collapse
boundary under $C_n:=\Lambda_n\Delta_n=\alpha_n\log n\to\infty$
(\cref{app:rank-boundary} for the contrary case), an
observable-specific refinement that does not change the calibration
principle.

\subsection{Counting function, upper-tail accumulation scale, and contact point}
\label{subsec:contact}

\begin{definition}[Cumulative gap-counting function]\label{def:gap-counting}
Set $z_n^*:=\max_{1\le j\le n}z_{n,j}$. The \emph{cumulative gap-counting function}
\[
 N_n(t):=\#\{j \mid z_n^* - z_{n,j} \le t\},\qquad t\ge 0,
\]
counts the indices $j$ with $z_n^*-z_{n,j}\le t$; such an index $j$
is called a \emph{$t$-competitor} (the prefix $t$ is dropped when
unambiguous). In the physics analogy with $-z_{n,j}$ as energies, the
gap $z_n^*-z_{n,j}$ is the excitation above the ground state $z=z_n^*$,
and $N_n(t)$ counts states with excitation energy at most $t$.
\end{definition}

The softmax probabilities of \cref{sec:notation} factor through the normalised partition function
 $Z_n(\beta):=\sum_{j=1}^n e^{-\beta(z_n^*-z_{n,j})}$,
$ p_{n,j}(\beta)=\exp[{-\beta(z_n^*-z_{n,j})}]/Z_n(\beta),$
and integration by parts (with $N_n(0^-)=0$ and $e^{-\beta t}N_n(t)\to 0$ as $t\to\infty$) gives
\begin{equation}\label{eq:Z-by-parts}
 Z_n(\beta)=\beta\int_0^\infty e^{-\beta t}\,N_n(t)\,dt.
\end{equation}
The integrand of \eqref{eq:Z-by-parts} is $e^{-\beta t}N_n(t)$. Whether the exponential damping $e^{-\beta t}$ outpaces the growth of $N_n(t)$ is captured by the envelope $N_n(t)\le e^{\beta t}$, and the smallest exponent for which this envelope holds is the next definition.

\begin{definition}[Upper-tail accumulation scale]\label{def:upper-crowding}
The \emph{upper-tail accumulation scale} $\Lambda_n$ is the smallest exponent for which the envelope
\begin{equation}\label{eq:Nn-exp-envelope}
 N_n(t)\le e^{\Lambda_n t}\qquad(t>0)
\end{equation}
holds, equivalently
\begin{equation}\label{eq:Lambda-def}
 \Lambda_n:=\sup_{t>0}\frac{\log N_n(t)}{t}.
\end{equation}
\end{definition}

Throughout the remainder of this section, we assume $\Lambda_n<+\infty$ eventually in $n$; the boundary case $\Lambda_n=\infty$ (ties at the maximum, which forces $\log N_n(t)/t\to\infty$ as $t\downarrow 0$) is deferred to \cref{cor:ties}. Under this finite-$\Lambda_n$ assumption the map $t\mapsto\log N_n(t)/t$ is non-increasing between consecutive competitor gaps and the supremum is attained as a maximum at some positive competitor gap $z_n^*-z_{n,j}>0$ (\cref{lem:discreteGamma}).

\begin{definition}[Contact point]\label{def:contact}
Call a positive competitor gap $u=z_n^*-z_{n,j}>0$ a \emph{contact gap}
if $\log N_n(u)/u=\Lambda_n$ (the envelope $N\le e^{\Lambda_n t}$ is
attained at $t=u$). Let $\Delta_n$ be the largest such gap, and define
the \emph{contact accumulation exponent}
$\alpha_n:=\log N_n(\Delta_n)/\log n$. The pair $(\alpha_n,\Delta_n)$
is the \emph{(largest) contact point} and satisfies the \emph{contact
formula}
\begin{equation}\label{eq:contact-sandwich}
 \Lambda_n=\frac{\alpha_n\log n}{\Delta_n}.
\end{equation}
\end{definition}

\par
\begin{wrapfigure}[16]{R}{0.42\linewidth}
\centering
\includegraphics[width=\linewidth]{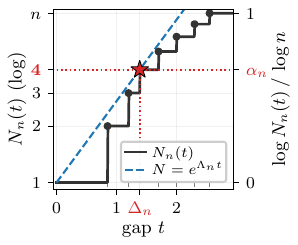}
\caption{$N_n(t)$ and envelope $N=e^{\Lambda_n t}$; contact point (red star). Sorted-rank version: \cref{fig:lambda-schematic-full}.}
\label{fig:lambda-schematic}
\end{wrapfigure}
\noindent
\Cref{fig:lambda-schematic} shows the gap-counting curve, the
exponential envelope, and the contact point. The coordinate
$\alpha_n$ records the finite-$n$ contact count on the scale of
$n$, while $\Delta_n$ records the corresponding gap. The contact formula \eqref{eq:contact-sandwich} is the only relation between
these coordinates used below.

\begin{theorem}\label{thm:deterministic-}
$(\Lambda_n)$ is the TE scale.
\end{theorem}

\begin{proof}
Set $s_n:=\beta_n/\Lambda_n$.

\textit{Subcritical direction.}
Let $\eta_n>0$ be the gap to the closest competitor, then $D_n(\beta)=(1-e^{-\beta\eta_n})/Z_n(\beta)$.
Suppose $s_n\to 0$. By \eqref{eq:contact-sandwich}, $N_n(\Delta_n)=n^{\alpha_n}$ and $\beta_n\Delta_n=s_n\alpha_n\log n$, so monotonicity of $e^{-\beta_n t}$ on the $N_n(\Delta_n)$ tokens with $z_n^*-z_{n,j}\le\Delta_n$ gives
\[
 Z_n(\beta_n)
 \ge N_n(\Delta_n)\,e^{-\beta_n\Delta_n}
 = n^{\alpha_n(1-s_n)}
 \ge n^{\alpha_n/2}
\]
for all large $n$ since $s_n\le 1/2$ eventually. Using $1-e^{-x}\le x$ and $\eta_n\le\Delta_n$, it holds that $D_n(\beta_n)\le\beta_n\eta_n/Z_n(\beta_n)\le\beta_n\Delta_n/n^{\alpha_n/2}=s_n\alpha_n\log n/n^{\alpha_n/2}\le(2/e)\,s_n\to 0$, where the last step uses $\sup_{a>0}a\log n/n^{a/2}=2/e$.

\textit{Supercritical direction.} $H_n(\beta_n) = \log Z_n(\beta_n)-\beta_n Z_n'(\beta_n)/{Z_n(\beta_n)}$. Suppose $s_n\to\infty$. The winner ($z_{n,j^*}=z_n^*$) contributes $1$, so $Z_n(\beta_n)\ge 1$. The envelope \eqref{eq:Nn-exp-envelope} and the integral form \eqref{eq:Z-by-parts} give
\begin{equation}\label{eq:Z-upper-supercritical}
 Z_n(\beta_n)\le\beta_n\int_0^\infty e^{-(\beta_n-\Lambda_n)t}\,dt=\frac{\beta_n}{\beta_n-\Lambda_n},
\end{equation}
which $\to 1$, so $Z_n(\beta_n)\to 1$. Now, differentiating \eqref{eq:Z-by-parts} gives \(-\beta Z_n'(\beta)=\beta^2\int_0^\infty t e^{-\beta t}N_n(t)\,dt-Z_n(\beta)\). Since \(0\le-\beta Z_n'(\beta)\), \(N_n(t)\le e^{\Lambda_n t}\), and \(Z_n(\beta)\ge1\), we have \(0\le-\beta Z_n'(\beta)\le(\beta/(\beta-\Lambda_n))^2-1\) for \(\beta>\Lambda_n\). At \(\beta=\beta_n\) this is \(o(1)\), so \(H_n(\beta_n)\to0\) follows from \(Z_n(\beta_n)\to1\).
\end{proof}

\begin{corollary}\label{cor:rank-boundary-diffuse}
$(\Lambda_n)$ is the RE scale if the contact-count entropy
$C_n:=\Lambda_n\Delta_n=\alpha_n\log n$ diverges.
\end{corollary}

A diverging contact-count entropy is what lets the rank-collapse
argument absorb the lower-order tail of $N_n$ and upgrade the TE
scale to the RE scale (proof and the rank-resolved boundary in the
contrary case in \cref{app:rank-boundary}).

\subsection{Logarithmic critical-scaling exponent}
\label{subsec:logarithmic-class}

Having established that $\Lambda_n$ is a critical scale, we now record how  its logarithmic growth rate is read from the contact coordinates of \cref{def:contact}. 

For any positive sequence $(X_n)$, if $X_n\asymp(\log n)^{\xi}$, then the exponent is unique when it exists: if $\xi$ and $\xi'$ both work, then
$(\log n)^{\xi-\xi'}$ is bounded above and below, forcing $\xi=\xi'$.
We call this $\xi$ the \emph{logarithmic growth exponent} of $(X_n)$
and denote it $\xi_X$, so that
\begin{equation}\label{eq:xi-X-def}
 X_n\asymp(\log n)^{\xi_X}.
\end{equation}
This single rule produces $\xi_\Lambda,\xi_\alpha,\xi_\Delta,\xi_\beta$
as instances of one family. The principal exponent of the system is
$\xi_\Lambda$, the growth exponent of the upper-tail accumulation
scale $\Lambda_n$.

\begin{corollary}[Exponent-selection rule]
\label{cor:exponent-selection}
Under $\Lambda_n\asymp(\log n)^{\xi_\Lambda}$ and
$\beta_n\asymp(\log n)^{\xi_\beta}$, any schedule avoiding both
top-two and entropy collapse has $\xi_\beta=\xi_\Lambda$ (otherwise
$\beta_n/\Lambda_n\to 0$ or $\to\infty$, contradicting
\cref{thm:deterministic-}).
\end{corollary}

The contact
formula \eqref{eq:contact-sandwich} expresses $\xi_\Lambda$ in the contact coordinates of \cref{def:contact} as a linear combination.

\begin{corollary}[Coordinate decomposition of $\xi_\Lambda$]
\label{cor:xi-decomposition}
If any two of $\Lambda_n,\alpha_n,\Delta_n$ admit logarithmic growth
exponents in the sense of \eqref{eq:xi-X-def}, then so does the third,
and they satisfy
\begin{equation}\label{eq:xi-coord-decomposition}
 \xi_\Lambda \;=\; \xi_\alpha - \xi_\Delta + 1.
\end{equation}
\end{corollary}

\begin{proof}
Multiplying $\asymp$-classes in the contact formula
\eqref{eq:contact-sandwich} gives
$
 \Lambda_n=\alpha_n\,(\log n)\,\Delta_n^{-1}
 \;\asymp\;(\log n)^{\xi_\alpha}\,(\log n)\,(\log n)^{-\xi_\Delta}
 =(\log n)^{\xi_\alpha-\xi_\Delta+1},
$
and uniqueness of the exponent yields
\eqref{eq:xi-coord-decomposition}. Solving the contact formula for
$\alpha_n$ or $\Delta_n$ covers the symmetric directions.
\end{proof}

The corollary makes $\xi_\alpha,\xi_\Delta$ first-class coordinate
exponents on equal footing with $\xi_\Lambda$, and the side of $1$ on
which $\xi_\Lambda$ lands is determined by the sign of
$\xi_\alpha-\xi_\Delta$. \cref{subsubsec:verification-coord-decomp}
uses this decomposition empirically.

\subsection{Gap-counting on the equicorrelated configuration and a block extension}
\label{subsec:gram-realizations}

The contact formula \eqref{eq:contact-sandwich} constrains the
one-row $N_n$ but not the configuration behind it. The deterministic
equicorrelated (``simplex'') configuration of
\citet{cowsik2024geometric}, whose critical $\beta_n\asymp\log n$ was
established by \citet{chen2025critical}, is the simplest configuration to
which the framework applies; gap-counting reads $\xi_\Lambda$ off it directly.
A balanced block-constant extension also tunes the contact count
$\alpha_n$. Both body constructions take the score matrix as a symmetric
Gram matrix $\Sigma_n$, equivalently the self-attention score matrix
\eqref{eq:sdp-attention} with $W_n^Q=W_n^K$
(\cref{prop:psd-gram-reduction}); this $Q=K$ restriction is standard in
theoretical analyses of attention
\citep{chen2025critical,geshkovski2023emergence,arnaboldi2025asymptotics}.
\citet{chen2025critical}, in particular, work with the equicorrelated
Gram of identity-coupled tokens, equivalent in our notation to the
$W_n^Q=W_n^K=I$ specialization of \eqref{eq:sdp-attention}. In both body constructions, $N_n(t)$ has at most two
positive jumps. We write $\delta_n$ for the minimum positive
gap and reserve $\Delta_n$ for the contact gap of
\cref{def:contact}; they coincide ($\Delta_n=\delta_n$) only in
the within-block-dominant case.

\paragraph{Equicorrelated (simplex) configuration.}
With common norm $\|x_i\|^2=q$ and pairwise inner product
$\langle x_i,x_j\rangle=\rho<q$, set $J_n:=\one\one^\top$,
$r_n^2:=q$, $\Delta_n:=q-\rho$ to obtain
\[
 \Sigma_n^{\mathrm{simp}}:=(r_n^2-\Delta_n)J_n+\Delta_n I_n\ge 0.
\]
Every query has flat competitor gap $\Delta_n$, so $N_n$ jumps from
$1$ to $n$ at $t=\Delta_n$ and $\Lambda_n=\log n/\Delta_n$,
$\alpha_n=1$, the full-count endpoint. The
$\beta_n\asymp\log n$ scaling of \citet{chen2025critical} is the
$\xi_\Lambda=1$ slice ($\Delta_n=\Theta(1)$); $\Delta_n=(\log n)^{1-\xi_\Lambda}$
realizes any $\xi_\Lambda\ge0$. Reading off the contact coordinates
gives $\xi_\alpha=0$ and $\xi_\Delta=1-\xi_\Lambda$, which satisfy the
decomposition \eqref{eq:xi-coord-decomposition} of
\cref{cor:xi-decomposition}: this configuration sits on the
$\xi_\alpha=0$ axis with $\xi_\Lambda$ entirely carried by
$\xi_\Delta$.

\begin{figure}[!t]
\centering
\includegraphics[width=0.9\linewidth]{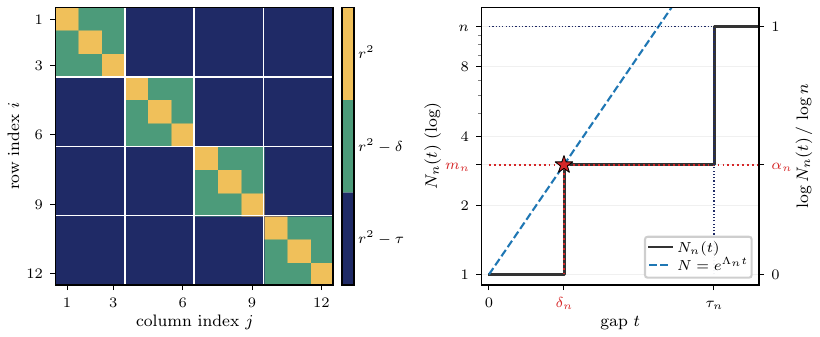}
\caption{The two-level block-constant configuration reads $\Lambda_n$ off the
within-block gap $\delta_n$ in the dominant regime
$\log m_n/\delta_n>\log n/\tau_n$. \emph{Left:} the three-valued
block-constant Gram matrix $\Sigma_n^{\mathrm{block}}$
\eqref{eq:two-level-gram} for $n=12$, $m_n=3$.
\emph{Right:} $N_n(t)$ jumps from $1$ to $m_n$ at $t=\delta_n$ and from
$m_n$ to $n$ at $t=\tau_n$; the dashed envelope $N=e^{\Lambda_n t}$ is
tangent at the contact point $(\Delta_n,N_n(\Delta_n))=(\delta_n,m_n)$
(red star), with secondary axis $\log N_n(t)/\log n$.}
\label{fig:two-level-block}
\end{figure}

\paragraph{Two-level block-constant extension.}
A balanced block-constant analogue of $\Sigma_n^{\mathrm{simp}}$
tunes $\alpha_n$ as well as $\xi_\Lambda$. Fix $m_n\in\{2,\dots,n\}$ with
$m_n\mid n$, and partition $\{1,\dots,n\}$ into blocks of size
$m_n$. Let $B_n$ be the block-membership matrix, equivalently the
block-diagonal matrix with $n/m_n$ diagonal blocks equal to
$J_{m_n}$. For $0<\delta_n<\tau_n\le r_n^2$, define
\begin{equation}\label{eq:two-level-gram}
 \Sigma_n^{\mathrm{block}}
 := \delta_n I_n+(\tau_n-\delta_n)B_n+(r_n^2-\tau_n)J_n\ \ge 0.
\end{equation}
Positivity follows from $I_n,B_n,J_n\ge 0$ and the nonnegative
coefficients. The entries are $r_n^2$ on the diagonal,
$r_n^2-\delta_n$ inside a block, and $r_n^2-\tau_n$ across blocks;
thus $r_n^2$ is a common score shift, while $\delta_n$ and $\tau_n$
are the two relevant row gaps. Consequently $N_n(t)=1$ for
$0\le t<\delta_n$, $N_n(t)=m_n$ for $\delta_n\le t<\tau_n$, and
$N_n(t)=n$ for $t\ge\tau_n$ (\cref{fig:two-level-block}). Hence
\begin{equation}\label{eq:two-level-Lambda}
 \Lambda_n
 =\max\!\left\{\frac{\log m_n}{\delta_n},\,\frac{\log n}{\tau_n}\right\}.
\end{equation}
The smaller-gap jump dominates when
$\log m_n/\delta_n>\log n/\tau_n$; then $\Delta_n=\delta_n$ and
$\alpha_n=\log m_n/\log n$. A prescribed $\Lambda_n=(\log n)^{\xi_\Lambda}$
is realized by taking $\delta_n=\log m_n/(\log n)^{\xi_\Lambda}$ and
$\tau_n=C\log n/(\log n)^{\xi_\Lambda}$ with $C>1$. With $m_n=n^\beta$
($\beta\in(0,1]$), the contact coordinates give $\xi_\alpha=0$ and
$\xi_\Delta=1-\xi_\Lambda$, recovering the simplex slice on a different
$\alpha_n=\beta$ level. With bounded $m_n=m\ge 2$,
$\alpha_n=\log m/\log n\asymp(\log n)^{-1}$ gives $\xi_\alpha=-1$, and
$\delta_n=\log m/(\log n)^{\xi_\Lambda}\asymp(\log n)^{-\xi_\Lambda}$
gives $\xi_\Delta=-\xi_\Lambda$; both branches satisfy
\eqref{eq:xi-coord-decomposition} (e.g.\ in the second branch
$\xi_\alpha-\xi_\Delta+1=-1-(-\xi_\Lambda)+1=\xi_\Lambda$). The block
extension is the simplest configuration in which $\xi_\alpha$ takes
values other than $0$. The reverse case has $\Delta_n=\tau_n$ and
$\alpha_n=1$, coinciding with the simplex endpoint. When $\Lambda_n$
is specified only up to $\asymp$, we pass to block-compatible
subsequences with $m_n\mid n$.

The block-constant configuration corresponds to token vectors with a
hierarchical cluster structure: within each block, tokens cluster
around a common centroid at within-block gap $\delta_n$, while distinct
blocks are separated by the larger across-block gap $\tau_n$. Such
hierarchically structured token representations are themselves an
object of study, with \citet{li2023how} giving a mechanistic account of
topic clustering inside trained transformers and
\citet{garnierbrun2025how} analyzing attention block structure under
hierarchical filtering of the input.

\subsection{Empirical gap-counting analysis}
\label{subsec:verification}

We apply the gap-counting construction of
\cref{def:gap-counting,def:upper-crowding,def:contact} on two
representative settings of \cref{sec:xi-sweep}: pretrained
Qwen-7B-Chat at extrapolation context lengths (inference-time) and
GPT-124M trained from scratch with a learnable per-position $\beta_n$
vector (training-time). The aim is to read $\hat\xi_\Lambda$ from
attention scores alone, align it with the $\xi_\beta^*$ optima of
\cref{sec:xi-sweep}, and verify the coordinate decomposition of
\cref{cor:xi-decomposition} on real attention rows.

\subsubsection{Aggregation protocol}
\label{subsubsec:verification-protocol}

\emph{Attention rows and exclusions.} Each tuple (layer, head,
document or validation batch, query position) determines an attention
row: a length-$n$ vector of scores $(z_{n,1},\dots,z_{n,n})$. From every
row, we read the contact triple $(\Lambda_n,\alpha_n,\Delta_n)$ via
\cref{def:gap-counting,def:upper-crowding,def:contact} at numerical
tolerance $\epsilon=10^{-6}$. Rows with $\log_{10}\Lambda_n>5$ are
exact-tie configurations ($N_n(0)\ge 2$) that saturate the
floating-point supremum in \eqref{eq:Lambda-def}; they are excluded
from every aggregate below.

\emph{Cells and $n$-grids.} Let $\mathcal C$ index the experimental
conditions: each \emph{cell} $c\in\mathcal C$ is a fixed
(setting, dataset, $\neval$) on Qwen and a fixed
(setting, dataset, training iteration) on nanoGPT. Each cell carries
a finite \emph{$n$-grid} $\mathcal N(c)\subset\mathbb N$, the set of
context lengths at which scores are sampled in that cell (values
listed in the next paragraph). For $c\in\mathcal C$ and $n\in\mathcal
N(c)$, let $  \mathcal R(c,n)
 :=\{\,\text{non-tied attention rows of cell $c$ at context length $n$}\,\}$
and write $K(c,n):=|\mathcal R(c,n)|$ for its cardinality.

\emph{Bucket-mean slope.} For $X\in\{\Lambda,\alpha,\Delta\}$ and a
row $i\in\mathcal R(c,n)$, write $X_n^{(i)}$ for the value of the
contact-triple coordinate $X$ read off row $i$ at context length
$n$. The \emph{bucket mean} of $\log X_n$ in cell $c$ at grid point
$n\in\mathcal N(c)$ is
$\overline{\log X}(c,n):=K(c,n)^{-1}\sum_{i\in\mathcal R(c,n)}\log X_n^{(i)}$.
The estimator $\hat\xi_X(c)$ is the OLS slope of
$\overline{\log X}(c,n)$ regressed on $\log\log n$ over
$n\in\mathcal N(c)$; setting $X=\Lambda,\alpha,\Delta$ yields the
three values $\hat\xi_\Lambda(c),\hat\xi_\alpha(c),\hat\xi_\Delta(c)$
reported in the body table.

\paragraph{Instantiation and intervals.}
We instantiate this on Qwen-7B-Chat over PG-19
\citep{rae2020compressive} and Proof-Pile-2
\citep{azerbayev2024llemma} at the two extrapolation context lengths
$\neval\in\{16384,32768\}$ of \cref{subsubsec:sweep-qwen} (4-point
$n$-grid per $\neval$, fit per-$\neval$ to keep the dynamic-NTK RoPE
rescale fixed within a fit; full Qwen acquisition in
\cref{sec:verification-extras}). On nanoGPT, GPT-124M is trained from
scratch with a learnable per-position $\beta_n$ at $\neval=1024$
(6-point grid $\{64,128,256,512,768,1024\}$, three evaluation seeds
pooled into $\mathcal R(c,n)$; configurations in
\cref{subsec:setup-gpt124m}, checkpoint selection and the
per-(layer, head) and rescaled-gap variants in
\cref{sec:nanogpt-gapcount-protocol}). Bracketed
intervals in \cref{tab:verif-xi-summary} (and tie\,\%) are half the
IQR of $B=200$ bootstrap resamples over the cell's natural unit set,
$(\ell,h)$ on Qwen and $(s,\ell,h)$ on nanoGPT.

\FloatBarrier
\subsubsection{Comparing the gap-counting reading with the \texorpdfstring{\cref{sec:xi-sweep}}{Section 3} sweeps}
\label{subsubsec:verification-coord-decomp}

\paragraph{Cell-by-cell agreement.}
\Cref{tab:verif-xi-summary} aligns $\hat\xi_\Lambda$ with the
\cref{sec:xi-sweep} optimum $\xi_\beta^*$ on eight
(setting, run, iter) cells. On Qwen, the per-$\neval$ PPL optima
$\xi^*_{\beta,\rm PPL}$ and $\hat\xi_\Lambda$ agree within the §3
$0.5$-grid (absolute gap $\le 0.26$), with both decreasing from
$\xi_\beta^*\in[2.31,2.50]$ at $\neval=16384$ to $[1.74,2.00]$ at
$\neval=32768$. On nanoGPT, $\hat\xi_\Lambda$ matches the
multiplier-side $\hat\xi_\beta$ on the same six-point domain to
$|\hat\xi_\Lambda-\hat\xi_\beta|\in\{0.12,0.14\}$ at
$\mathrm{iter}=25\mathrm{k}$ and $\{0.17,0.27\}$ at
$\mathrm{iter}=100\mathrm{k}$. The two columns are read off disjoint
inputs (attention scores vs.\ the trained $\beta_n$ vector);
\cref{thm:deterministic-} predicts equality only asymptotically,
consistent with the residual on the nanoGPT $n\le 1024$ window.

\begin{figure}[!t]
\centering
\includegraphics[width=0.33\linewidth]{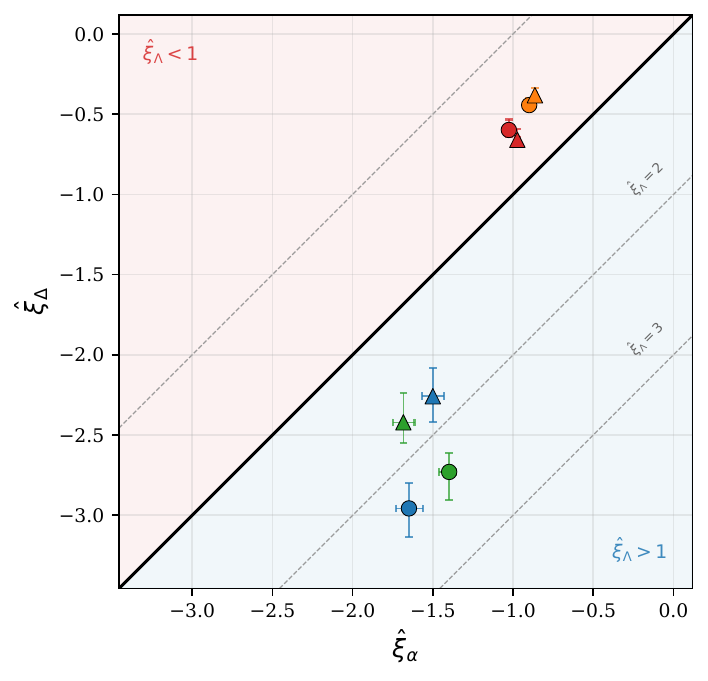}\hfill
\includegraphics[width=0.33\linewidth]{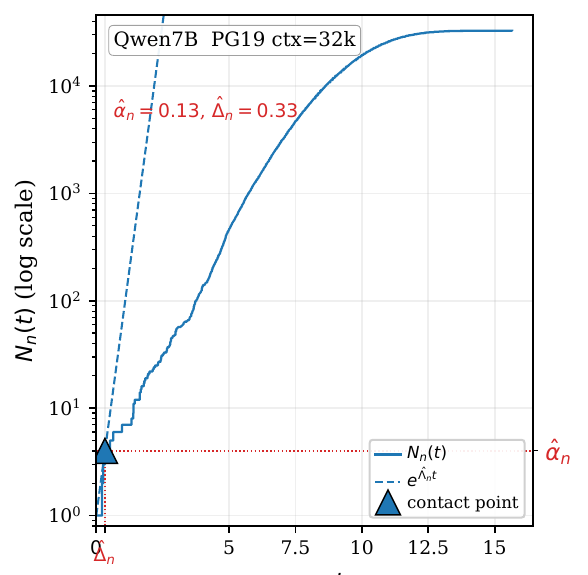}\hfill
\includegraphics[width=0.33\linewidth]{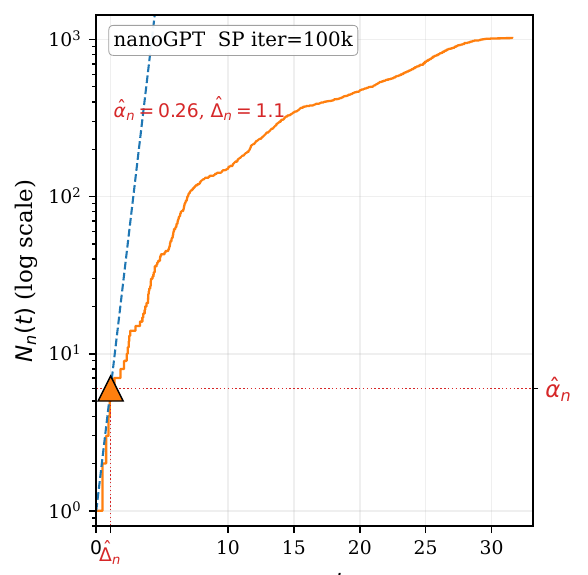}\\[0.4em]
\includegraphics[width=0.85\linewidth]{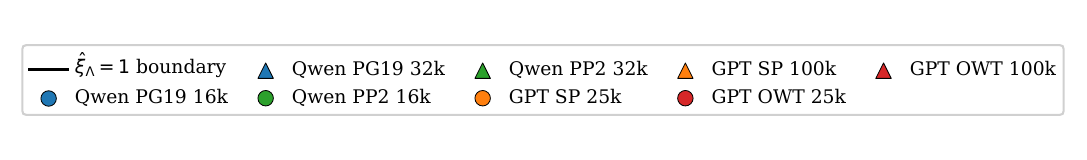}
\caption{(Left) Qwen and nanoGPT cells land on opposite wedges of the
$\hat\xi_\Lambda=1$ boundary $\hat\xi_\alpha=\hat\xi_\Delta$, so the
side of $1$ each setting reaches is fixed by the sign of
$\hat\xi_\alpha-\hat\xi_\Delta$ via \eqref{eq:xi-coord-decomposition}.
(Middle, right) The two sides recast as $N_n(t)$: Qwen's contact lies
near the origin on a steep $\Lambda_n$-envelope, nanoGPT's further out
on a shallow one.}
\label{fig:xi-coord-decomp}
\end{figure}

\begin{table}[!t]
\centering
\footnotesize
\setlength{\tabcolsep}{3pt}
\begin{tabular}{l l c c c c c c}
\toprule
setting & dataset & $\neval$/iter & tie\,\%
   & $\xi_\beta^*$ (\cref{sec:xi-sweep})
   & $\hat\xi_\Lambda$
   & $\hat\xi_\alpha$
   & $\hat\xi_\Delta$ \\
\midrule
Qwen7B   & PG19 & $16384$ & $7.4\,(\pm0.2)$ & $2.5$               & $2.31\,(\pm0.10)$ & $-1.65\,(\pm0.08)$ & $-2.96\,(\pm0.17)$ \\
Qwen7B   & PG19 & $32768$ & $8.8\,(\pm0.2)$ & $2.0$               & $1.76\,(\pm0.10)$ & $-1.50\,(\pm0.07)$ & $-2.26\,(\pm0.17)$ \\
Qwen7B   & PP2  & $16384$ & $6.8\,(\pm0.2)$ & $2.5$               & $2.33\,(\pm0.11)$ & $-1.40\,(\pm0.05)$ & $-2.73\,(\pm0.15)$ \\
Qwen7B   & PP2  & $32768$ & $7.2\,(\pm0.2)$ & $2.0$               & $1.74\,(\pm0.11)$ & $-1.68\,(\pm0.07)$ & $-2.42\,(\pm0.16)$ \\
\midrule
nanoGPT  & SP   & $25\mathrm{k}$  & $2.8\,(\pm0.1)$ & $0.40\,(\pm0.03)$ & $0.54\,(\pm0.03)$ & $-0.90\,(\pm0.01)$ & $-0.44\,(\pm0.04)$ \\
nanoGPT  & SP   & $100\mathrm{k}$ & $6.0\,(\pm0.4)$ & $0.35\,(\pm0.04)$ & $0.52\,(\pm0.03)$ & $-0.86\,(\pm0.01)$ & $-0.38\,(\pm0.04)$ \\
nanoGPT  & OWT  & $25\mathrm{k}$  & $4.0\,(\pm0.4)$ & $0.46\,(\pm0.05)$ & $0.57\,(\pm0.04)$ & $-1.03\,(\pm0.02)$ & $-0.60\,(\pm0.05)$ \\
nanoGPT  & OWT  & $100\mathrm{k}$ & $5.7\,(\pm0.3)$ & $0.42\,(\pm0.05)$ & $0.69\,(\pm0.05)$ & $-0.97\,(\pm0.02)$ & $-0.66\,(\pm0.05)$ \\
\bottomrule
\end{tabular}
\caption{$\hat\xi_\Lambda$, read off attention scores alone, lands on the same side of $\xi=1$ as the sweep optimum $\xi_\beta^*$ in every cell: both exceed $1$ on the four Qwen cells, both fall below $1$ on the four nanoGPT cells.
Each $\hat\xi_X$ ($X\in\{\Lambda,\alpha,\Delta\}$) is the OLS slope of the bucket-mean $\overline{\log X}(c,n)$ regressed on $\log\log n$ across the cell's $n$-grid (\cref{subsubsec:verification-protocol}); $\xi_\beta^*$ is the per-cell optimum from \cref{sec:xi-sweep}.
Bracketed intervals are half-IQRs of $B=200$ bootstrap resamples over $(\ell,h)$ on Qwen and $(s,\ell,h)$ on nanoGPT, with the same convention for tie\,\%, the fraction of exact-tie rows ($N_n(0)\ge 2$, excluded from aggregates).}
\label{tab:verif-xi-summary}
\end{table}

\paragraph{Coordinate decomposition.}
The decomposition \eqref{eq:xi-coord-decomposition} of
\cref{cor:xi-decomposition} carries through to the bucket-mean
slopes, and the right two columns of \cref{tab:verif-xi-summary}
confirm $\hat\xi_\Lambda=\hat\xi_\alpha-\hat\xi_\Delta+1$ row by row
at the table's precision; the side of $1$ on which $\hat\xi_\Lambda$
lands is fixed by the sign of $\hat\xi_\alpha-\hat\xi_\Delta$.
\Cref{fig:xi-coord-decomp} (left) plots the eight
$(\hat\xi_\alpha,\hat\xi_\Delta)$ pairs of
\cref{tab:verif-xi-summary}: the four Qwen extrapolation-time points
sit below the diagonal $\hat\xi_\alpha=\hat\xi_\Delta$
($|\hat\xi_\Delta|>|\hat\xi_\alpha|$, contact gap shrinks faster
with $\log\log n$), and the four nanoGPT training-time points sit
above it ($|\hat\xi_\alpha|>|\hat\xi_\Delta|$, contact count shrinks
faster). The middle and right panels recast one representative row
per contrast cell (Qwen PG-19 $\neval=32768$ at $n=32768$; nanoGPT
SlimPajama $\mathrm{iter}=100\mathrm{k}$ at $n=1024$) as $N_n(t)$.

\paragraph{Calibration character of $\Lambda_n$.}
\Cref{tab:verif-xi-summary} evaluates $\Lambda_n$ as a calibration
scale for inverse temperature: agreement with $\xi_\beta^*$ relies
only on \cref{cor:exponent-selection} and does not require TE and RE
scales to coincide. The rank-side checks in
\cref{subsec:empirical-rank-boundary} are a sanity check that the
stronger coordinate-wise observable is not pathological
(\cref{cor:rank-boundary-diffuse} on the four nanoGPT cells; the
rank-resolved substitute and a direct $p_n^*$ sweep on the four Qwen
cells with bounded $C_n$).

\section{Discussion and Conclusion}
\label{sec:discussion-conclusion}
This paper asks what determines the critical scaling of inverse temperature for self-attention in the long-context limit.
Our answer is that it is determined by $(\Lambda_n, \alpha_n, \Delta_n)$.
The cumulative gap-counting function $N_n(t)$ of \cref{def:gap-counting}
records how many competitors lie within each gap of the winning score,
and the upper-tail accumulation scale $\Lambda_n$ of \cref{def:upper-crowding}
gives the relevant scale. \Cref{thm:deterministic-} identifies
$\Lambda_n$ as a critical scale for arbitrary deterministic score rows;
\cref{cor:rank-boundary-diffuse} sharpens this to the standard
rank-collapse boundary whenever the contact-count entropy
$\alpha_n\log n=\Lambda_n\Delta_n$ diverges, with empirical verification
on the bounded-$C_n$ cells via \cref{subsec:empirical-rank-boundary}. Thus the exponent
$\xi_\beta$ in $\beta_n\asymp(\log n)^{\xi_\beta}$ is not universal:
any non-degenerate schedule has $\xi_\beta=\xi_\Lambda$
(\cref{cor:exponent-selection}), the logarithmic growth rate of
$\Lambda_n$ in the sense of \eqref{eq:xi-X-def}. Equivalently, the contact formula~\eqref{eq:contact-sandwich} shows that $\xi_\beta$ is determined
by both the contact count $\alpha_n$ and the contact gap
$\Delta_n$ of \cref{def:contact}, namely
$\xi_\beta=\xi_\alpha-\xi_\Delta+1$ (\cref{cor:xi-decomposition}).
This also explains why the mechanisms in \cref{tab:priorwork} give
different exponents (they induce different $N_n$, not different
universal laws), and why empirical long-context behavior can prefer
values away from the legacy logarithmic choice $\xi_\beta=1$.

\paragraph{Limitations and open problem.}
The main limitation of the present framework is that the score family, and hence \((\Lambda_n,\alpha_n,\Delta_n)\), is taken as given. We do not yet derive it from the architecture, positional encoding, data distribution, or training dynamics of a concrete transformer. A natural open problem is therefore to endogenize these quantities through the dynamics of self-attention. Recent inference-time theories of recurrent or infinitely deep self-attention describe token representations as interacting particles whose empirical geometry can develop clustered, metastable, or lower-dimensional structure across depth or loop time \citep{geshkovski2023emergence,geshkovski2024dynamic,karagodin2024clustering,bruno2025multiscale,tomihari2025recurrent,rigollet2026meanfield}. Complementary learning-time analyses show that training can progressively encode semantic, positional, or hierarchical structure into embeddings and attention maps \citep{li2023how,garnierbrun2025how,arnaboldi2025asymptotics}. Coupling such inference-time and learning-time dynamics with the present gap-counting theory would turn \(N_n(t)\), \(\Lambda_n\), and \((\alpha_n,\Delta_n)\) into observables that depend on depth, loop time, and training time, making it possible to predict how the critical inverse-temperature scale is shaped by the model and data rather than treating the score family as fixed.

\section*{Acknowledgments}
T.H.\ was supported by JST BOOST Grant Number JPMJBY24G4.
The authors are grateful to Ken M.\ Nakanishi for sharing details of the experimental setup for SSMax.

\bibliographystyle{hunsrtnatarxiv}
\bibliography{reference}

\newpage
\appendix
\section{Experimental setup}
\label{sec:experimental-setup}

This appendix collects the experimental basic setup shared across the
inference-time $\xi$-sweep of \cref{sec:xi-sweep} and the training-time
gap-exponent trajectory of \cref{subsubsec:verification-protocol}. The
per-experiment grids, seeds, $\xi_0$ values, and numerical tolerances
are recorded in the corresponding protocol appendices
(\cref{sec:xi-sweep-numerics,sec:nanogpt-gapcount-protocol,subsec:verification-protocol}).

\subsection{Datasets}
\label{subsec:setup-datasets}

The Qwen long-context $\xi$-sweep
(\cref{subsubsec:sweep-qwen}) draws documents from the
PG-19 \citep{rae2020compressive} test split and the Proof-Pile-2
\citep{azerbayev2024llemma} arXiv test split (up to $20$ per dataset
per $\neval$; the per-cell counts entering the perplexity averages
are listed in \cref{subsec:body-error-bars}). The Qwen-7B-Chat
gap-counting analysis (\cref{subsec:verification}) uses a smaller
fixed subset of these documents because the per-row analysis is
heavier: $4$ from PG-19 and $5$ from Proof-Pile-2 enter the headline
statistics of \cref{tab:verif-xi-summary}, which together with the
$(\xi_0,\neval,\text{ratio})$ grid produces the
$774{,}144$ attention rows enumerated in \cref{sec:verification-extras}. The
GPT-124M training-time runs of
\cref{subsubsec:sweep-nanogpt,subsubsec:verification-protocol} are
pretrained on SlimPajama \citep{soboleva2023slimpajama} and
OpenWebText \citep{gokaslan2019openwebtext}, one run per dataset.

\subsection{Qwen long-context evaluation}
\label{subsec:setup-qwen}

We probe two pretrained checkpoints: Qwen-1.8B-Chat with training
window $\ntrain=8192$, and Qwen-7B-Chat ($32$ transformer layers, $32$
attention heads per layer) also with $\ntrain=8192$. Both keep rotary
position embeddings \citep{su2024roformer} on with the dynamic-NTK
RoPE base schedule of \citet[\S 3.4]{peng2023yarn} defined just below.
On top of dynamic-NTK, we apply the legacy length-dependent multiplier
\[
  m(n;\xi)=\bigl(\max(1,\log n/\log \ntrain)\bigr)^\xi
\]
to the attention logits at extrapolation lengths $n>\ntrain$, equivalent
to the parameterization $\beta_n\propto(\log n)^\xi$ for $n>\ntrain$ and
to the no-multiplier baseline at $\xi=0$; the YaRN inverse-temperature
window $\xi=2$ \citep{peng2023yarn} corresponds to the YaRN schedule
as derived in \cref{sec:yarn-derivation}. The contexts $n$ and
$\xi$-grids vary by experiment and are recorded with the respective
figures.

\paragraph{Dynamic-NTK RoPE.}\label{para:dynamic-ntk-rope}
Standard RoPE \citep{su2024roformer} rotates the per-head query/key
pair in $\mathbb R^d$ using a fixed base $\theta_0$ (typically
$\theta_0=10000$): the $k$-th coordinate pair at position $i$ is
multiplied by $R(i\,\theta_0^{-2k/d})$. The \emph{dynamic-NTK} schedule
of \citet[\S 3.4]{peng2023yarn} (also \citealt{bai2023qwen} for Qwen)
replaces $\theta_0$ by a per-evaluation base $\theta_0\,s(\neval)$
where, with the extrapolation ratio
$\rho(\neval):=\max(1,\neval/\ntrain)$,
\[
 s(\neval)\;:=\;\rho(\neval)^{d/(d-2)}.
\]
The base is recomputed once per inference window of size $\neval$ and
then held fixed across positions $i\le\neval$ inside the window. The
high-frequency rotations are nearly unchanged while the
low-frequency ones are stretched, so embeddings of relative offsets
that exceed $\ntrain$ at evaluation time stay within the angular range
the model saw during training. Throughout this paper, ``dynamic-NTK
RoPE'' refers to this schedule, with the legacy multiplier $m(n;\xi)$
above applied as a separate inverse-temperature scalar on the attention
logits.

\paragraph{Sliding-window perplexity.}
We report \emph{back-half sliding-window perplexity}: a long document
is partitioned into non-overlapping windows of length $\neval$, and
within each window only the second half (positions $j\in(\neval/2,\neval]$)
contributes to the loss, so every scored token is conditioned on
$\ge\neval/2$ preceding tokens. Concretely, with $\ell(j)$ the
cross-entropy of the model on the $j$-th token of the window, the
window contribution to log-perplexity is
$\sum_{j>\neval/2}\ell(j)\bigm/(\neval/2)$, averaged across windows
and documents and exponentiated. This is the protocol called
\emph{strict sliding-window perplexity} in earlier drafts and in some
external references; we drop the modifier in the body since the
back-half restriction is the only convention used. The convention
controls position-context confounding at extrapolation lengths
$\neval>\ntrain$: under no-overlap windows the early positions of a
window see only short prefixes, which would otherwise dominate the
loss for $n>\ntrain$.

\subsection{GPT-124M training-time setup}
\label{subsec:setup-gpt124m}

GPT-124M \citep{radford2019language} ($12$ transformer layers, $12$
attention heads per layer, context length $\neval=1024$) is trained from
scratch using the nanoGPT codebase \citep{karpathy2022nanogpt} in four
configurations: two positional encodings (learned absolute, as in
GPT-2 \citep{radford2019language}; rotary \citep{su2024roformer})
$\times$ two pretraining datasets (SlimPajama, OpenWebText). Each
configuration runs for $100{,}000$ iterations with checkpoints saved
every $5000$ iterations; the seven milestones
$\mathrm{iter}\in\{0,5\textrm{k},10\textrm{k},25\textrm{k},
50\textrm{k},75\textrm{k},100\textrm{k}\}$ are the snapshots used by
the body and appendix figures.

\paragraph{Learnable inverse-temperature vector $\beta_n$.}
Each model carries an additional length-$\neval$ vector
$\beta_n$ (one parameter per context position $n$) tied across layers
and heads; per (layer, head) the actual attention multiplier is
\begin{equation}\label{eq:setup-betan-realization}
  \beta_{n,\ell,h}=s_{\ell,h}\,\beta_n+b_{\ell,h},
\end{equation}
with per-head scale $s_{\ell,h}$ and bias $b_{\ell,h}$ trained
alongside $\beta_n$. The $n$-dependence is carried entirely by
$\beta_n$; the body's $\beta_n$ regression fit
(\cref{subsubsec:sweep-nanogpt}) reads $\beta_n$ off after training
and fits it to $a_1(\log n)^\xi$, while the gap-exponent diagnostic
(\cref{subsubsec:verification-protocol}) compares raw and effective
gap exponents row-wise. Unlike SSMax, which fixes
$\beta_{n,\ell,h}=s_{\ell,h}\log n$ with a single scalar $s_{\ell,h}$
per head, here $\beta_n$ is unconstrained at initialization and is
shaped only by the training loss. The original SSMax design-rationale
experiment of \citet{nakanishi2025scalable} used the same
unconstrained $\beta_n$ vector and fit it by least squares to
$a_1\log n+a_2$, motivating the logarithmic constraint.

\paragraph{Hyperparameters.}
\Cref{tab:gpt124m-config} records the architecture, $\beta_n$
realization, and optimization settings shared across the four
configurations, together with the four (positional encoding, dataset)
pairs that distinguish them. All values are read from the training
configs of the runs (\texttt{n\_layer}, \texttt{n\_head},
\texttt{n\_embd}, optimizer schedule, etc.).

\begin{table}[!htbp]
\centering
\small
\begin{tabular}{l l}
\toprule
\multicolumn{2}{l}{\emph{Architecture (4 runs shared, GPT-2 124M)}} \\
\midrule
Layers $L$                     & $12$ \\
Attention heads $H$            & $12$ \\
Embedding dim $d_{\rm model}$  & $768$ \\
Per-head dim $d_{qk}=d_v$      & $64$ \\
FFN dim $d_{\rm ff}$           & $3072$ ($=4\,d_{\rm model}$) \\
Vocabulary size                & $50{,}257$ (GPT-2 BPE; padded to $50{,}304$) \\
Context length $\neval$        & $1024$ \\
Linear bias                    & enabled \\
Dropout                        & $0$ \\
Total parameters               & ${\sim}124$M \\
\midrule
\multicolumn{2}{l}{\emph{Inverse-temperature $\beta_n$ realization}} \\
\midrule
Per-position factor $\beta_n$  & learnable vector of length $\neval$, init.\ $\beta_n[n]=1$ \\
Per-(layer,head) realization   & $\beta_{n,\ell,h}=s_{\ell,h}\,\beta_n+b_{\ell,h}$ (\cref{eq:setup-betan-realization}) \\
$s_{\ell,h}, b_{\ell,h}$       & per-head, learnable; $s_{\ell,h}$ init.\ $1$, $b_{\ell,h}$ init.\ $0$ \\
\midrule
\multicolumn{2}{l}{\emph{Optimization (4 runs shared)}} \\
\midrule
Optimizer                      & AdamW \\
Learning-rate schedule         & cosine: $6{\times}10^{-4}\to6{\times}10^{-5}$ \\
Warmup                         & $1000$ iters \\
Weight decay                   & $0.1$ \\
$(\beta_1,\beta_2)$            & $(0.9, 0.95)$ \\
Gradient clip                  & $1.0$ \\
Per-step batch                 & $32$ sequences $\times$ grad-accum $64$ (effective $2048$) \\
Precision                      & bfloat16 \\
Iterations                     & $100{,}000$ \\
Checkpoint cadence             & every $5000$ iters (milestone keep) \\
\midrule
\multicolumn{2}{l}{\emph{Configurations (4 runs)}} \\
\midrule
1: learned absolute PE  & SlimPajama \\
2: learned absolute PE  & OpenWebText \\
3: RoPE ($\theta_0=10000$) & SlimPajama \\
4: RoPE                 & OpenWebText \\
\bottomrule
\end{tabular}
\caption{GPT-124M training-time configuration. The four runs share
the architecture, $\beta_n$ realization, and optimization settings
listed above; they differ only in the positional encoding (learned
absolute as in \citealt{radford2019language}, or RoPE
\citealp{su2024roformer}) and the pretraining dataset (SlimPajama
\citealp{soboleva2023slimpajama}, or OpenWebText
\citealp{gokaslan2019openwebtext}). Three evaluation seeds
$\{1337,1338,1339\}$ produce independent validation batches at every
checkpoint.}
\label{tab:gpt124m-config}
\end{table}

\subsection{Compute resources}
\label{subsec:compute-resources}

All experiments ran on academic GPU clusters equipped with
NVIDIA H100 GPUs ($96$~GB HBM3). Two job sizes were used: a
single-GPU job (one H100) and a full-node job (eight H100s).

\paragraph{Qwen inference (\cref{subsec:setup-qwen,subsec:verification}).}
The Qwen long-context $\xi_\beta$ sweep
(\cref{sec:xi-sweep,fig:ppl-xi-7b-pg19,fig:ppl-xi-7b-pp})
and the gap-counting attention-row capture
(\cref{subsubsec:verification-protocol}) each run as single-GPU jobs.
Per back-half sliding-window perplexity job (one $(\xi_\beta,\text{dataset})$
covering all five contexts $\neval\in\{8192,12288,16384,24576,32768\}$
on $20$ documents): ${\sim}25$ minutes wall on one H100, of which
${\sim}5$ minutes per context. Per attention-row gap capture job
(one $(\xi_0,\text{dataset},\neval,\text{doc})$): ${\sim}2$ minutes
wall on one H100. The full $\xi_\beta$ refinement to $0.1$ resolution
on Qwen-7B-Chat used $\le 60$ GPU-hours. The gap-counting capture for
the body Table~2 cells, the $\xi_0$-sweep, and the $\xi_0$
fixed-point refinement together used ${\sim}10$ additional GPU-hours
on the same single-H100 budget. Total Qwen H100-hours for the
content reported in this paper: ${\sim}80$.

\paragraph{nanoGPT training-time runs (\cref{subsec:setup-gpt124m,sec:nanogpt-gapcount-protocol}).}
Each of the four training-time GPT-124M runs of
\cref{tab:gpt124m-config} runs as a full-node job
($8$ H100, distributed via \texttt{torchrun} with effective batch $2048$
sequences of length $1024$, walltime budget $24$~hours per run).
A complete run reaches $100{,}000$ iterations in well under the
budget; total nanoGPT training compute reported in this paper is bounded by
$4\,\text{runs}\times 24\,\text{h}\times 8\,\text{H100} = 768$
H100-hours. Per-checkpoint dump
and post-processing for the seven milestone checkpoints
$\{0,5\textrm{k},\ldots,100\textrm{k}\}$ on three evaluation seeds
adds order $20$ GPU-hours.

\paragraph{Storage.}
The Qwen full-gap safetensors at the body's $(ds,\neval,\xi_0,\mathrm{doc})$
grid plus the off-grid points generated for the $\xi_0$ fixed-point
search occupy ${\sim}150$~GB of project-area storage; the nanoGPT
checkpoint set occupies ${\sim}30$~GB per training run.

\paragraph{Preliminary work.}
The headline compute above does not include preliminary $\xi_\beta$-
sweep iterations on Qwen-1.8B-Chat at intermediate grid resolutions,
exploratory $\xi_0$ ablations, and discarded gap-capture pipelines
(direct top-$k$ vs full sorted-gap dumps). Counting these, the full
research project consumed roughly twice the headline budget, all on
the same internal H100 clusters.

\subsection{Reproducibility capsule}
\label{subsec:reproducibility-capsule}

This subsection collects the environment and artifact pointers
needed to reproduce the headline experiments alongside the
protocols documented in
\cref{subsec:setup-qwen,subsec:setup-gpt124m,sec:nanogpt-gapcount-protocol,sec:verification-extras,subsec:body-error-bars}.

\paragraph{Environment.}
We list only the dependencies that affect the experiment outputs;
the SciPy-stack analysis packages used downstream of the captured
CSVs are omitted because any current version reproduces the same
numerics. The Qwen long-context evaluations run under Python~$3.11$
with PyTorch~$2.6$ (CUDA~$12.4$ build),
\texttt{transformers~$4.32.0$} (the version pinned by the
\texttt{Qwen-7B-Chat} model card and required by the bundled
custom modeling file loaded with \texttt{trust\_remote\_code=True}),
\texttt{flash-attn~$2.8.3$}, and \texttt{einops~$0.8$} (used inside
that modeling file). The GPT-124M training-time runs use
Python~$3.12$ with the same PyTorch build and
\texttt{tiktoken~$0.12$} (GPT-2 BPE). All experiments run on
NVIDIA~H100 ($96$~GB HBM3) GPUs.

\paragraph{Artifact identifiers.}
The Qwen checkpoints are the public \texttt{Qwen/Qwen-7B-Chat} and
\texttt{Qwen/Qwen-1\_8B-Chat} HuggingFace releases used as published
(no fine-tuning); their bundled \texttt{config.json} reports
\texttt{transformers\_version}~$4.32.0$, fixing the modeling-file
revision the \texttt{trust\_remote\_code} loader executes. The
training-time runs are built on the public nanoGPT codebase
\citep{karpathy2022nanogpt} (snapshot from spring~$2026$) with the
length-$\neval$ learnable $\beta_n$ vector inserted as in
\eqref{eq:setup-betan-realization}. The text corpora are the
\texttt{deepmind/pg19}, \texttt{EleutherAI/proof-pile-2} (arXiv
subset), \texttt{cerebras/SlimPajama-627B}, and
\texttt{Skylion007/openwebtext} HuggingFace dataset releases as of
the experiment dates; tokenization uses each model's bundled
tokenizer (Qwen's BPE for the Qwen evaluations, GPT-2 BPE via
\texttt{tiktoken} for the nanoGPT training and validation passes).

\paragraph{Modification surface.}
The Qwen-side per-token multiplier
$m(n;\xi)=\bigl(\max(1,\log n/\log \ntrain)\bigr)^\xi$ of
\cref{subsec:setup-qwen} is applied after the per-head dot-product
score
$z_{n,j}=\langle\mathbf{q}_{n,i_n},\mathbf{k}_{n,j}\rangle/\sqrt{d_{qk}}$
and before the softmax inside the modified Qwen forward pass; the
unrescaled $z_{n,j}$ that the gap-counting analysis reads is the
score before this multiplier and softmax are applied. Dynamic-NTK
RoPE is the upstream Qwen schedule with the per-$\neval$ base
rescale of \citet[\S 3.4]{peng2023yarn}, controlled by Qwen's
existing \texttt{use\_dynamic\_ntk} configuration switch. The
nanoGPT-side modification is a single length-$\neval$ per-position
scalar vector $\beta_n$ initialised to one and trained jointly with
the rest of the model; the per-(layer, head) affine
$\beta_{n,\ell,h}=s_{\ell,h}\beta_n+b_{\ell,h}$ then folds into the
attention computation as in \eqref{eq:setup-betan-realization}.

\paragraph{Procedure cross-references.}
Document selection, query-position rule, bucket and $n$-grid
choices, and bootstrap procedures are documented in the sections
referenced above; the per-row attention-score acquisition pipeline
is in \cref{sec:verification-extras}, and the reported $\pm$
half-widths are derived as in \cref{subsec:body-error-bars}.

\section{Detailed protocols and aggregate numerics for the $\xi_\beta$-sweeps}
\label{sec:xi-sweep-numerics}

This appendix records the body $\beta_n$ regression-fit protocol and
the training-time trajectory deferred from \cref{sec:xi-sweep}.

\subsection{Learnable inverse-temperature vector: $\beta_n$ regression-fit protocol and trajectory}
\label{subsec:nanogpt-pn-trajectory}

The four-configuration training schedule and the per-(layer, head)
realization \eqref{eq:setup-betan-realization} of the learnable vector
$\beta_n$ are as in \cref{subsec:setup-gpt124m}. At each of the seven
milestone checkpoints, we extract $\beta_n$ (a length-$1024$ vector)
and fit the bias-free regression model $\beta_n=a_1(\log n)^{\xi_\beta}$ in
two settings:

\begin{center}
\small
\setlength{\tabcolsep}{4pt}
\begin{tabular}{lccc}
\hline
& Domain & Bias $a_2$ & $\xi_\beta$ search \\
\hline
SSMax \citep{nakanishi2025scalable} (= M1) & $n\in[1,1024]$ & free & grid \\
Ours, grid (= M2) & $n\in\{64,\ldots,1024\}$ & $\equiv 0$ & grid (step $0.05$) \\
Ours, OLS slope & $n\in\{64,128,\ldots,1024\}$ (six points) & $\equiv 0$ & continuous \\
\hline
\end{tabular}
\end{center}

The body figures \cref{fig:nanogpt-pn-sp,fig:nanogpt-pn-owt} use the
grid version of Ours. \Cref{fig:nanogpt-pn-trajectory} reports the
M2 loss landscape on the same fine grid for all four runs at three
training milestones; the rotary runs are added here only as a
reference, since the body intentionally does not contrast positional
encodings.

\paragraph{Why $a_2\equiv 0$.}
The per-(layer, head) realization
\eqref{eq:setup-betan-realization} already places a layer-side scalar
bias inside the inverse-temperature parameterization. A free $a_2$ in
the post-hoc fit would re-introduce that same bias on top, leaving
$\xi_\beta$ under-identified by the residual curvature of $\beta_n$.
Empirically, on the trained checkpoints, M1 ($a_2$ free) selects
$\xi_\beta=1$ in only $3/28\;(11\%)$ snapshots and otherwise wanders the
grid with no stable trend across iters; the bias absorbs the
signal. Forcing $a_2\equiv 0$ removes this absorption and lets $\xi_\beta$
identify the residual log-log slope of $\beta_n$.

\paragraph{Why $n\ge 64$.}
The trimming threshold matches the smallest $n$ used by
the body's gap-counting estimator $\hat\xi_\Lambda$ of
\cref{def:gap-counting,def:contact}, namely
$n=64$ (the smallest of the six $n$ values in
\cref{sec:nanogpt-gapcount-protocol}). Defining $\hat\xi_\beta$ on the
same domain makes the comparison with $\hat\xi_\Lambda$ direct rather
than a domain-mismatch artefact; the residual finite-$n$ gap
$|\hat\xi_\Lambda-\hat\xi_\beta|$ is the value reported per cell in
\cref{tab:verif-xi-summary}. The first half of the context length
$n\le 50$ also contains the small-$n$ dynamics of $\beta_n$ that
deviate from the asymptotic power family; trimming there isolates the
scaling regime that \cref{cor:xi-decomposition} characterizes.

\paragraph{Trim sensitivity.}
\Cref{tab:n-trim-sens} reports $\hat\xi_\beta$ as the trim threshold
varies, on the two absolute-PE runs at $\mathrm{iter}=100\mathrm{k}$.
The $n\ge 64$ choice sits in a stable plateau (within $\pm 0.04$ of the
$n\ge 51$ value), and pushing the threshold to $n\ge 200$ shifts
$\hat\xi_\beta$ by $\le 0.16$ on the more sensitive dataset
(OpenWebText), consistent with $\beta_n$ being well approximated by
the power family on the asymptotic tail.

\begin{table}[h]
\centering
\small
\begin{tabular}{lcccccc}
\hline
trim & $n\ge 10$ & $n\ge 25$ & $n\ge 51$ & $n\ge 64$ & $n\ge 100$ & $n\ge 200$ \\
\hline
SlimPajama  & 0.305 & 0.336 & 0.368 & 0.380 & 0.410 & 0.477 \\
OpenWebText & 0.375 & 0.414 & 0.457 & 0.474 & 0.517 & 0.618 \\
\hline
\end{tabular}
\caption{$\hat\xi_\beta$ from the continuous OLS fit of $\beta_n$ to
$a_1(\log n)^{\xi_\beta}$ on $n\ge\text{trim}$, two absolute-PE
training-time GPT-124M runs at $\mathrm{iter}=100\mathrm{k}$. The
body uses the discrete six-point variant on
$n\in\{64,128,256,512,768,1024\}$, which gives $\hat\xi_\beta=0.346$
(SP) and $0.418$ (OWT).}
\label{tab:n-trim-sens}
\end{table}

\begin{figure}[t]
\centering
\includegraphics[width=\linewidth]{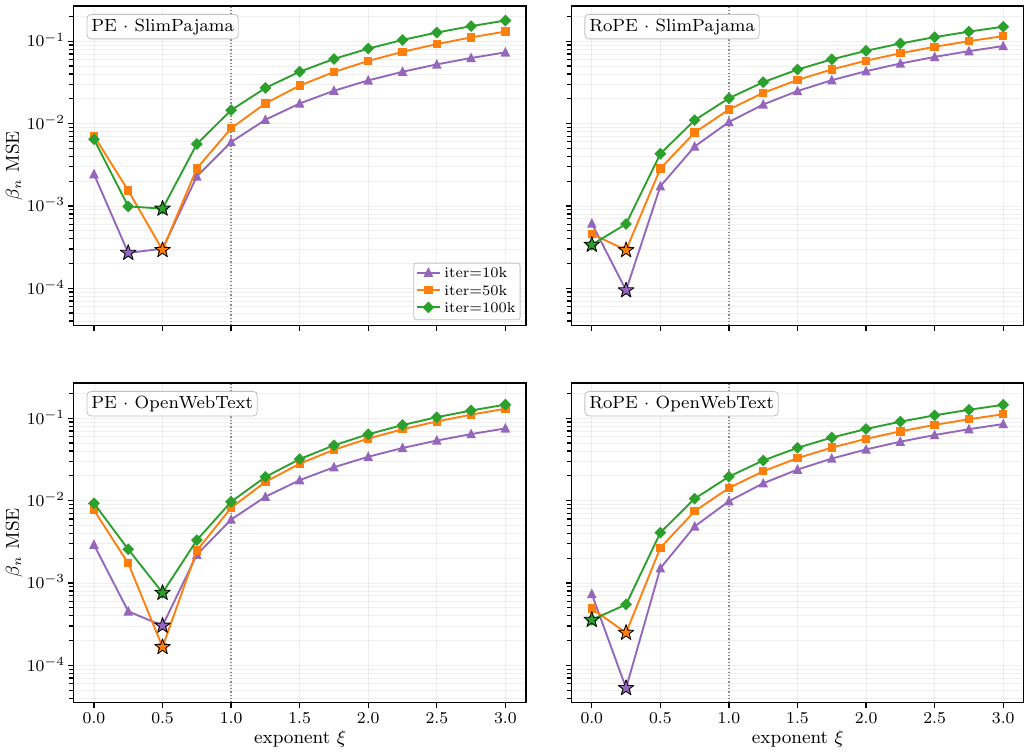}
\caption{M2 loss landscape on all four training-time GPT-124M runs.
The $y$-axis is the per-point MSE of the bias-free post-hoc
least-squares fit of $\beta_n$ to $a_1(\log n)^{\xi_\beta}$ on $n\ge 64$, on
a log scale; one curve per snapshot
$\mathrm{iter}\in\{10\mathrm{k},50\mathrm{k},100\mathrm{k}\}$, with
stars at per-curve argmin and a dotted vertical reference at the
logarithmic value $\xi_\beta=1$. PE and RoPE runs settle at distinct
exponents within each dataset.}
\label{fig:nanogpt-pn-trajectory}
\end{figure}

\subsection{Error-bar widths for the body four-panel figure}
\label{subsec:body-error-bars}

The cell-wise relative upper bounds quoted in
\cref{fig:body-four-panel} (SE on PPL $\le 10\%$ of the cell mean in
(a, b); half-IQR on MSE $\le 8\%$ of the point-estimate MSE in
(c, d)) are derived by the protocol below. Each width is the natural
sampling uncertainty of the y-axis quantity at that panel type,
reduced to a single per-panel number by taking the maximum across the
cells (Qwen) or $\xi_\beta$ grid points (nanoGPT) plotted in the panel.

\paragraph{Panels (a, b): doc-level standard error on PPL.}
At each $(\xi_\beta,\neval)$ cell, the body figure plots the per-document
strict perplexity $\mathrm{PPL}_{\rm strict}^{(d)}$ averaged over the
$n_{\rm doc}$ documents that fit at that $\neval$, with
$n_{\rm doc}\in\{14,15,\dots,19\}$ on PG-19 and $n_{\rm doc}=20$ on
Proof-Pile-2 (\cref{subsec:setup-qwen}). The reported quantity is the
doc-level standard error of the cell mean
$\mathrm{SE}=s/\sqrt{n_{\rm doc}}$, where $s$ is the sample standard
deviation of $\{\mathrm{PPL}_{\rm strict}^{(d)}\}_{d=1}^{n_{\rm doc}}$;
the cell-wise maxima of $\mathrm{SE}/\mathrm{mean}$ across the
$35$ cells per dataset are $9.52\%$ on PG-19 and $8.79\%$ on
Proof-Pile-2 (both attained at $\xi_\beta=0$, the length-multiplier-off
endpoint). The cell-wise medians are $7.49\%$ and $8.04\%$
respectively, so the body's $10\%$ envelope is loose by less than a
factor of $1.3$. A per-cell document-resampling bootstrap (B=200,
percentile half-IQR of the cell mean) gives a tighter $4.95\%$ /
$5.32\%$ but is dominated by the more conservative SE-based number
quoted in the body caption.

\paragraph{Panels (c, d): residual bootstrap of the M2 MSE curve.}
At each $\xi_\beta$ on the body grid, the MSE plotted in (c, d) is a
deterministic function of a single trained $\beta_n$ vector
(\cref{subsec:nanogpt-pn-trajectory}); its sampling uncertainty arises
from the regression residuals on positions $n\in[64,1024]$. Let
$\hat a_1$ denote the M2 estimate at the given $\xi_\beta$ and
$e_n=\beta_n-\hat a_1(\log n)^{\xi_\beta}$ the residual. We resample
$e_n$ with replacement across positions $B=200$ times; for each
resample $e_n^*$ we form
$y_n^*=\hat a_1(\log n)^{\xi_\beta}+e_n^*$, refit $a_1$, and recompute
the per-point MSE. The width is the half-IQR of these $B$ MSE values
divided by the point-estimate MSE. Across all six (dataset, iter)
snapshots (\cref{fig:nanogpt-pn-trajectory}), the per-grid maximum of
this ratio peaks at $7.29\%$ (SlimPajama, $\mathrm{iter}=50\mathrm{k}$,
near $\xi_\beta\!\approx\!0.55$); medians stay between $2.7\%$ and
$2.9\%$.

\section{Auxiliary results for the deterministic theory}
\label{sec:envelope-norm-obstruction}

This appendix collects strands deferred from \cref{sec:theory}.
\cref{subsec:discrete-representation} reduces the supremum
\eqref{eq:Lambda-def} defining $\Lambda_n$ to a finite maximum over
competitor gaps.
\cref{subsec:ties} handles the degenerate case of ties at the maximum.
\cref{subsec:score-realizability} gives the score-realizability reduction
used throughout \cref{subsec:gram-realizations}: any positive score matrix
arises from self-attention with $Q_n=K_n$.

Throughout this appendix we write
$\mathcal U_n:=\{z_n^*-z_{n,j}:1\le j\le n,\ z_{n,j}<z_n^*\}$
for the set of positive competitor gaps.

\subsection{Discrete representation of $\Lambda_n$}
\label{subsec:discrete-representation}

\begin{lemma}[Discrete representation of $\Lambda_n$]\label{lem:discreteGamma}
On the finite-$\Lambda_n$ domain ($N_n(0)=1$), the upper-tail accumulation scale of \cref{def:upper-crowding} satisfies
\[
 \Lambda_n
 =\sup_{t>0}\frac{\log N_n(t)}{t}
 =\max_{u\in\mathcal U_n}\frac{\log N_n(u)}{u}.
\]
If $N_n(0)\ge 2$ (ties at the maximum), then $\Lambda_n=\infty$.
\end{lemma}

\begin{proof}
$t\mapsto N_n(t)$ is right-continuous and piecewise constant with jumps only at points of $\{0\}\cup\mathcal U_n$, so $\log N_n(t)$ is constant on each interval $[u,u')$ between consecutive points of $\{0\}\cup\mathcal U_n\cup\{+\infty\}$, while the denominator $t$ is increasing. Hence $t\mapsto\log N_n(t)/t$ is nonincreasing on each such interval, and the supremum is attained at some $u\in\mathcal U_n$. If $N_n(0)\ge 2$, $N_n(t)\ge 2$ for every $t>0$ and $\log N_n(t)/t\to\infty$ as $t\downarrow 0$.
\end{proof}

\paragraph{Sorted-rank rendering.}
The body figure \cref{fig:lambda-schematic} plots the
counting-function view $t\mapsto N_n(t)$. The same data admits an
equivalent rendering as the sorted competitor gaps
$j\mapsto g_{n,(j)}$, obtained by swapping the two axes
$j\leftrightarrow N$: the right-continuous step $N_n(t)$ on log $y$
maps to the ascending staircase $g_{n,(j)}$ on log $x$, and the
exponential envelope $N=e^{\Lambda_n t}$ inverts to its logarithmic
form $t=\log j/\Lambda_n$. The contact point
$(\Delta_n,N_n(\Delta_n))$ on the right is the same point
$(j^*,\Delta_n)$ with $j^*=N_n(\Delta_n)$ on the left, and
\cref{lem:discreteGamma}'s claim that the supremum is attained at a
competitor gap is the statement that the sorted-rank staircase
touches the envelope on its riser, not its tread.
\Cref{fig:lambda-schematic-full} shows both renderings together.

\begin{figure}[ht]
\centering
\includegraphics[width=\linewidth]{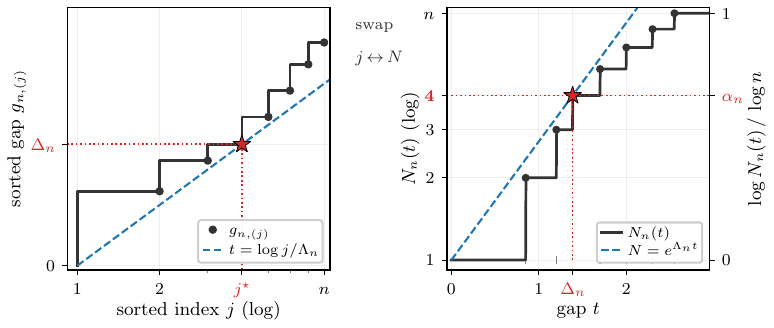}
\caption{Two equivalent renderings of the same $N_n$.
\emph{Right:} the cumulative gap-count curve $t\mapsto N_n(t)$ on the gap axis $t\ge 0$ and a log $y$-axis, with the dashed exponential envelope $N=e^{\Lambda_n t}$ \eqref{eq:Nn-exp-envelope}; the red star marks the contact point $(\Delta_n,N_n(\Delta_n))$, and the secondary right axis reads $\alpha_n=\log N_n(\Delta_n)/\log n$. This is the panel reproduced in the body as \cref{fig:lambda-schematic}.
\emph{Left:} the equivalent sorted-competitor-gap rendering $j\mapsto g_{n,(j)}$ on a log $j$-axis, with $t=\log j/\Lambda_n$ as the inverse form of the accumulation envelope; the red star sits at the same contact point with horizontal coordinate $j^*=N_n(\Delta_n)$. The two panels are related by the axis swap $j\leftrightarrow N$.}
\label{fig:lambda-schematic-full}
\end{figure}

\subsection{Ties at the maximum}
\label{subsec:ties}

\begin{corollary}[Ties at the maximum]\label{cor:ties}
Let
\[
 n_{\max}:=N_n(0)=\#\{j\in\{1,\dots,n\}:z_{n,j}=z_n^*\}.
\]
Then:
\begin{enumerate}
\item if $n_{\max}\ge 2$, then $\Lambda_n=\infty$;
\item for every $\beta>0$,
\[
 \max_j p_{n,j}(\beta)\le \frac{1}{n_{\max}},
 \qquad
 H_n(\beta)\ge \log n_{\max}.
\]
\end{enumerate}
In particular, if $n_{\max}\ge 2$ for infinitely many $n$, then no finite TE scale exists.
\end{corollary}

\begin{proof}
If $n_{\max}\ge 2$, then $z_n^*-z_{n,j}=0$ for the $n_{\max}$ tokens at the maximum, so $N_n(t)\ge n_{\max}$ for every $t>0$. Consequently
\[
 \frac{\log N_n(t)}{t}\ge \frac{\log n_{\max}}{t}\to\infty\qquad(t\downarrow 0),
\]
so $\Lambda_n=\infty$ by \cref{def:upper-crowding}. The top $n_{\max}$ softmax coordinates are equal, and their total mass is at most $1$, so each of them is at most $1/n_{\max}$. Thus $\max_j p_{n,j}(\beta)\le 1/n_{\max}$. Since every softmax coordinate is at most $1/n_{\max}$,
\[
 H_n(\beta)
 =-
 \sum_{j=1}^n p_{n,j}(\beta)\log p_{n,j}(\beta)
 \ge
 -\sum_{j=1}^n p_{n,j}(\beta)\log(1/n_{\max})
 =\log n_{\max}.
\]
For the final assertion, suppose $(X_n)\subset(0,\infty)$ is a finite TE scale and set $\beta_n:=nX_n$, so $\beta_n/X_n=n\to\infty$. Then \cref{def:coarse-critical-scale} forces $H_n(\beta_n)\to 0$, contradicting $H_n(\beta_n)\ge\log n_{\max}\ge\log 2$ on the infinite subsequence with $n_{\max}\ge 2$.
\end{proof}

\subsection{Score realizability via positive Gram matrices}
\label{subsec:score-realizability}

\begin{definition}[Positive matrix]
\label{def:positive-matrix}
A matrix $A\in\R^{n\times n}$ is \emph{positive semidefinite} if $x^\top A x\ge 0$
for every $x\in\R^n$. We call $A$ \emph{positive}, and write $A\ge 0$, if $A$ is
symmetric and positive semidefinite. Every Gram matrix considered in the sequel is
positive in this sense, so we use the predicate $A\ge 0$ throughout.
\end{definition}

\begin{proposition}[Gram reduction for positive score matrices]
\label{prop:psd-gram-reduction}
Let $\Sigma_n\in\R^{n\times n}$ satisfy $\Sigma_n\ge 0$ with $\rank \Sigma_n\le \dimqkn$. Then
there exist matrices $Q_n,K_n\in\R^{\dimqkn\times n}$ with
$Q_n=K_n$ such that the self-attention score matrix equals $\Sigma_n$:
\[
 \frac{\langle \qvec_{n,i},\kvec_{n,j}\rangle}{\sqrt{\dimqkn}}=(\Sigma_n)_{ij},
 \qquad 1\le i,j\le n .
\]
In particular, every row of $\Sigma_n$ is realized as a fixed-query self-attention score vector.
\end{proposition}

\begin{proof}
Choose $B_n\in\R^{\dimqkn\times n}$ with $B_n^\top B_n=\Sigma_n$, padding by zero rows if
necessary. Setting $Q_n=K_n=\dimqkn^{1/4}B_n$ gives
$Q_n^\top K_n/\sqrt{\dimqkn}=B_n^\top B_n=\Sigma_n$.
\end{proof}

\section{Rank-collapse boundary and resolved accumulation}
\label{app:rank-boundary}

The body's TE scale $\Lambda_n$ controls the top-two/entropy pair via
\cref{thm:deterministic-} but does not by itself imply coordinate-wise
rank collapse (\cref{subsec:collapse}). This appendix records the
exact rank-collapse criterion, defines the corresponding \emph{rank
boundary}, identifies the positive-resolution accumulation proxy that
sandwiches it, and proves \cref{cor:rank-boundary-diffuse}.

Define the \emph{rank free energy}
\begin{equation}\label{eq:rank-free-energy}
  \mathcal F_n(\beta):=\log Z_n(\beta).
\end{equation}

\begin{proposition}[Exact coordinate-wise rank-collapse criterion]\label{prop:exact-rank-criterion}
Let
\[
  \mathcal G_n(\beta):=\max_{i,j}\bigl|p_{n,i}(\beta)-p_{n,j}(\beta)\bigr|.
\]
For any positive deterministic sequence $(\beta_n)$,
\[
  \mathcal G_n(\beta_n)\to 0
  \quad\Longleftrightarrow\quad
  \mathcal F_n(\beta_n)\to\infty.
\]
Equivalently,
\[
  \max_j\left|p_{n,j}(\beta_n)-\frac{1}{n}\right|\to 0
  \quad\Longleftrightarrow\quad
  \mathcal F_n(\beta_n)\to\infty.
\]
\end{proposition}

\begin{proof}
The largest unnormalised softmax weight equals $1$, so
\[
  p_n^*(\beta):=\max_j p_{n,j}(\beta)=\frac{1}{Z_n(\beta)}=e^{-\mathcal F_n(\beta)}.
\]
Since $\mathcal G_n(\beta)\le p_n^*(\beta)$, the condition
$\mathcal F_n(\beta_n)\to\infty$ implies $\mathcal G_n(\beta_n)\to 0$.
Conversely, $\min_j p_{n,j}(\beta)\le 1/n$, hence
\[
  p_n^*(\beta)
  \le
  \mathcal G_n(\beta)+\frac{1}{n}.
\]
Thus $\mathcal G_n(\beta_n)\to 0$ implies $p_n^*(\beta_n)\to 0$, which is
equivalent to $\mathcal F_n(\beta_n)\to\infty$.

The coordinate-wise uniform-law equivalence follows from
\[
  p_n^*(\beta)\to 0
  \quad\Longleftrightarrow\quad
  \max_j\left|p_{n,j}(\beta)-\frac{1}{n}\right|\to 0,
\]
because $1/n\to 0$ and $0\le p_{n,j}(\beta)\le p_n^*(\beta)$.
\end{proof}

\begin{definition}[Rank boundary]\label{def:rank-boundary}
For $0\le r\le \log n$, define
\[
  B_n^{\mathrm{rank}}(r)
  :=
  \sup\bigl\{\beta\ge 0:\mathcal F_n(\beta)\ge r\bigr\}.
\]
\end{definition}

\Cref{prop:exact-rank-criterion} implies that coordinate-wise rank
collapse holds along $(\beta_n)$ if and only if there exists
$r_n\to\infty$, $r_n\le\log n$, such that
\[
  \beta_n\le B_n^{\mathrm{rank}}(r_n)
\]
eventually.

\begin{definition}[Resolved accumulation scale]\label{def:resolved-crowding}
For $r\ge 0$, define
\[
  \Lambda_n^{(r)}
  :=
  \sup_{t>0:\,\log N_n(t)>r}
  \frac{\log N_n(t)-r}{t},
\]
with the convention $\sup\emptyset=0$.
\end{definition}

At zero resolution,
\[
  \Lambda_n^{(0)}=\Lambda_n.
\]
Thus the upper-tail accumulation scale of the body is the
\emph{zero-resolution accumulation scale}. The rank-collapse boundary is
obtained by taking a positive diverging resolution $r_n\to\infty$,
which removes finite near-tie clusters from the envelope.

\begin{proposition}[Resolved accumulation implies rank-gap collapse]\label{prop:resolved-implies-rank}
Let $r_n\to\infty$. If
\[
  \beta_n<\Lambda_n^{(r_n)}
\]
eventually, then $\mathcal G_n(\beta_n)\to 0$.
\end{proposition}

\begin{proof}
By the definition of $\Lambda_n^{(r_n)}$, for all large $n$ there
exists $t_n>0$ such that
\[
  \beta_n t_n<\log N_n(t_n)-r_n.
\]
The $N_n(t_n)$ tokens within gap $t_n$ each contribute at least
$e^{-\beta_n t_n}$ to the partition function. Hence
\[
  Z_n(\beta_n)
  \ge
  N_n(t_n)\,e^{-\beta_n t_n}
  >
  e^{r_n}.
\]
Since $r_n\to\infty$, \cref{prop:exact-rank-criterion} gives
$\mathcal G_n(\beta_n)\to 0$.
\end{proof}

Define the Laplace envelope
\begin{equation}\label{eq:laplace-envelope}
  S_n(\beta):=\sup_{t\ge 0}\bigl\{\log N_n(t)-\beta t\bigr\}.
\end{equation}

\begin{lemma}[Laplace-envelope sandwich]\label{lem:laplace-sandwich}
For every $\beta>0$,
\[
  S_n(\beta)
  \le
  \mathcal F_n(\beta)
  \le
  S_n(\beta)+\log(1+\log n).
\]
\end{lemma}

\begin{proof}
For the lower bound, fix $t\ge 0$. The $N_n(t)$ tokens with gap at
most $t$ each contribute at least $e^{-\beta t}$, hence
\[
  Z_n(\beta)\ge N_n(t)\,e^{-\beta t}.
\]
Taking the supremum over $t$ gives $S_n(\beta)\le \mathcal F_n(\beta)$.

For the upper bound, write the ordered competitor gaps as
$0=t_{n,1}\le t_{n,2}\le\cdots\le t_{n,n}$. Since $N_n(t_{n,j})\ge j$, the
definition of $S_n$ implies
\[
  \log j-\beta t_{n,j}\le S_n(\beta).
\]
Thus
\[
  e^{-\beta t_{n,j}}\le \frac{e^{S_n(\beta)}}{j}.
\]
Summing over $j$ gives
\[
  Z_n(\beta)
  \le
  e^{S_n(\beta)}\sum_{j=1}^n\frac{1}{j}
  \le
  e^{S_n(\beta)}\,(1+\log n).
\]
Taking logarithms proves the upper bound.
\end{proof}

Consequently, for $r>\log(1+\log n)$,
\[
  \Lambda_n^{(r)}
  \le
  B_n^{\mathrm{rank}}(r)
  \le
  \Lambda_n^{(r-\log(1+\log n))}.
\]
Therefore, at resolutions $r_n\gg\log\log n$, the exact rank boundary
is described by the resolved accumulation scale up to the envelope
discretisation error $\log(1+\log n)$.

\begin{corollary}[Contact-count entropy criterion for rank collapse]\label{cor:diffuse-contact-rank}
Assume $\beta_n/\Lambda_n\to 0$. If the contact-count entropy
$C_n:=\log N_n(\Delta_n)=\Lambda_n\Delta_n$ satisfies
$C_n\to\infty$, then $\mathcal G_n(\beta_n)\to 0$.
\end{corollary}

\begin{proof}
Set $s_n=\beta_n/\Lambda_n$. The contact gap $\Delta_n$ gives
\[
  Z_n(\beta_n)
  \ge
  N_n(\Delta_n)\,e^{-\beta_n\Delta_n}
  =
  \exp\bigl\{(1-s_n)\,C_n\bigr\}.
\]
Since $s_n\to 0$ and $C_n\to\infty$, the right-hand side
diverges. \Cref{prop:exact-rank-criterion} applies.
\end{proof}

\begin{proof}[Proof of \cref{cor:rank-boundary-diffuse}]
The supercritical implication is the supercritical direction of
\cref{thm:deterministic-} and uses no contact-count hypothesis. The
subcritical implication is exactly \cref{cor:diffuse-contact-rank}.
\end{proof}

\begin{example}[Finite contact count: TE and RE scales separate]\label{ex:rank-fails}
The condition $C_n\to\infty$ in \cref{cor:rank-boundary-diffuse}
cannot be dropped: the score family below has $C_n=\log 2$ (bounded)
and rank collapse fails, even though $(\Lambda_n)$ is the TE scale.
Set
\[
 z_{n,1}=0,\qquad
 z_{n,2}=-\frac{\log 2}{\log n},\qquad
 z_{n,j}=-\log n\ \ (j\ge 3),
\]
so that $z_n^*=0$ and the competitor gaps are
$z_n^*-z_{n,2}=(\log 2)/\log n$ and
$z_n^*-z_{n,j}=\log n$ for $j\ge 3$. Then $\Lambda_n=\log n$,
$\Delta_n=(\log 2)/\log n$, and $\alpha_n=(\log 2)/\log n\to 0$. With
$\beta_n=\sqrt{\log n}$ so that $\beta_n/\Lambda_n\to 0$, the softmax
row converges to $(\tfrac12,\tfrac12,0,\dots,0)$: top-two collapse
holds, but the limit is not the uniform law and rank collapse fails.
A rank-resolved analysis is given in
\cref{ex:finite-contact-separation}.
\end{example}

\begin{example}[Rank-resolved analysis of the finite-contact case]\label{ex:finite-contact-separation}
For the score family of \cref{ex:rank-fails},
$\Lambda_n=\log n$ because the contact gap $\Delta_n=(\log 2)/\log n$
attains the zero-resolution envelope. The two competitor gap levels
satisfy $\log N_n((\log 2)/\log n)=\log 2$ and $\log N_n(\log n)=\log n$.
For any $r_n\to\infty$ with $r_n=o(\log n)$, eventually $r_n>\log 2$,
so the contact gap $\Delta_n$ is excluded by the constraint
$\log N_n(t)>r_n$ in \cref{def:resolved-crowding}, leaving only
$t=\log n$:
\[
  \Lambda_n^{(r_n)}
  =
  \frac{\log n-r_n}{\log n}
  \sim 1.
\]
Thus the TE scale is order $\log n$, while the rank-resolved
accumulation scale is order one. With $\beta_n=\sqrt{\log n}$,
$Z_n(\beta_n)\to 2$ and
$\mathcal G_n(\beta_n)\to\tfrac12$.
\end{example}

\subsection{Empirical instances on the verification cells}
\label{subsec:empirical-rank-boundary}

The verification cells of \cref{tab:verif-xi-summary} occupy two
distinct regimes of \cref{cor:rank-boundary-diffuse}.

\paragraph{TE-to-RE upgrade.}
\Cref{thm:deterministic-} identifies $\Lambda_n$ as the TE scale
unconditionally; the upgrade to RE scale via
\cref{cor:rank-boundary-diffuse} requires the contact-count entropy
$C_n=\alpha_n\log n$ to diverge, equivalently $\xi_\alpha>-1$. The
four nanoGPT cells of \cref{tab:verif-xi-summary} have
$\hat\xi_\alpha\in(-1.04,-0.86)$, so the relaxed condition holds with
$1+\hat\xi_\alpha>0$ giving a slow $C_n\asymp(\log n)^{1+\hat\xi_\alpha}$
divergence. The four Qwen cells have $\hat\xi_\alpha\in(-1.69,-1.40)$
on the observed window; since $C_n\ge\log 2$ always, the bucket-mean
slope $\hat\xi_\alpha<-1$ cannot persist asymptotically, and the
analytical RE-scale upgrade of \cref{cor:rank-boundary-diffuse} does
not directly apply, with the rank-resolved boundary of
\cref{prop:exact-rank-criterion} as the formal substitute. The
empirical boundary character of $\Lambda_n$ on these cells is then
verified directly below.

\paragraph{Boundary check on \texorpdfstring{$\Lambda_n$}{Lambda-n}.}
Independently of the analytical sufficient condition, the boundary
character of $\Lambda_n$ can be tested by reading the per-row maximum
softmax weight $p_n^*(\beta)=1/Z_n(\beta)$ at $\beta=\gamma\Lambda_n$
on a sweep around criticality. \Cref{tab:verif-pmax-boundary} reports
the median $p_n^*$ at $\gamma\in\{0.5,1.0,1.5\}$ and the per-cell
fraction of rows with $p_n^*\le 1/\log n$ at $\gamma=0.1$
(strict subcritical) on the four Qwen cells. Across all cells the
median traces a smooth monotone curve from $\sim 10^{-3}$ at
$\gamma=0.1$ to $\sim 1$ at $\gamma=4$, crossing $\sim 0.42$ at
$\gamma=1$, with $\sim 90\%$ of rows below $1/\log n$ at $\gamma=0.1$
and $0\%$ at $\gamma\ge 1$. The framework's $\Lambda_n$ separates the
subcritical-collapse regime from the supercritical-concentration
regime empirically, including on the Qwen cells where the contact-
count entropy is bounded and \cref{cor:rank-boundary-diffuse} does not
directly apply. \Cref{ex:rank-fails} shows that this analytical
upgrade is sharp in the worst case (a contrived score family with
$C_n$ bounded fails rank collapse), but the Qwen rows in the bounded-
$C_n$ regime do not realize that pathology at the operating $\beta$.

\begin{table}[t]
\centering
\footnotesize
\setlength{\tabcolsep}{4pt}
\begin{tabular}{l l c c c c}
\toprule
& & \multicolumn{3}{c}{median $p_n^*$ at $\beta=\gamma\Lambda_n$}
  & frac.\ $p_n^*\le 1/\log n$ \\
\cmidrule(lr){3-5}\cmidrule(lr){6-6}
dataset & $\neval$ & $\gamma=0.5$ & $\gamma=1.0$ & $\gamma=1.5$
   & at $\gamma=0.1$ \\
\midrule
PG-19        & $16384$ & $0.20$ & $0.45$ & $0.64$ & $89\%$ \\
PG-19        & $32768$ & $0.17$ & $0.42$ & $0.63$ & $91\%$ \\
Proof-Pile-2 & $16384$ & $0.18$ & $0.42$ & $0.63$ & $89\%$ \\
Proof-Pile-2 & $32768$ & $0.15$ & $0.40$ & $0.62$ & $93\%$ \\
\bottomrule
\end{tabular}
\caption{Boundary check via $p_n^*(\gamma\Lambda_n)$ on Qwen-7B-Chat
attention rows at $\xi_0=2$ ($45{,}467$ rows total across the four
cells, exact-tie rows excluded). Median $p_n^*$ at $\gamma\in
\{0.5,1.0,1.5\}$ and per-cell fraction with $p_n^*\le 1/\log n$ at
$\gamma=0.1$. On the wider sweep $\gamma\in\{0.1,0.25,0.5,0.75,1.0,
1.5,2.0,4.0\}$ the median $p_n^*$ is monotone in $\gamma$ on every
cell, ranging from $\sim\!0.001$ at $\gamma=0.1$ to $\sim\!0.96$ at
$\gamma=4$, and the same monotonicity holds on the nanoGPT cells.}
\label{tab:verif-pmax-boundary}
\end{table}

\section{Derivation of the YaRN inverse-temperature scaling}
\label{sec:yarn-derivation}

This appendix spells out why the YaRN context-window extension
of \citet{peng2023yarn} produces an effective inverse-temperature
scaling $\beta_n\asymp(\log n)^2$, i.e.\ implicit exponent $\xi=2$ in
the parameterisation of \cref{subsec:prior-comparison}.

\citet{peng2023yarn} replace the standard softmax logit
$q_m^{\top}k_n/\sqrt{d}$ with $q_m^{\top}k_n/(t\sqrt{d})$ for a
length-dependent temperature $t=t(s)$, where $s=n/\ntrain$ is the ratio of
the target context length $n$ to the training window $\ntrain$. The
recommended schedule (their \S 3.3) is
\[
 \tfrac{1}{\sqrt{t(s)}}\;=\;1+0.1\,\ln s,
\]
and the multiplier $1/\sqrt{t(s)}$ is \emph{multiplied onto the rotary
embeddings of both the query and the key}. Substituting into the logit
gives a rescaling
\[
 \beta_n
 \;\propto\;
 1/t(n/\ntrain)
 \;=\;
 \bigl(1+0.1\ln(n/\ntrain)\bigr)^2
 \;\asymp\;(\log n)^2,
\]
i.e.\ a scalar multiplier $\asymp(\log n)^2$ asymptotically. The fact
that the multiplier is applied to both sides of the inner product, rather
than once, is what doubles the exponent relative to a
$\log n$-only temperature, and places YaRN in the $\xi=2$ row of
\cref{tab:priorwork}.

\section{Training-time exponent measurement: acquisition protocol, per-seed granularity, and limitations}
\label{sec:nanogpt-gapcount-protocol}

This appendix gives the per-row gap-counting acquisition protocol,
the $\beta_n$ regression-fit protocol, and limitations for the
training-time consistency check of
\cref{subsubsec:verification-protocol}. The upstream training setup
that produces the learnable inverse-temperature vector $\beta_n$
being measured here, namely the models, datasets, four-configuration
schedule, and the per-(layer, head) realization
\eqref{eq:setup-betan-realization}, is documented in
\cref{subsec:setup-gpt124m}; the body $\beta_n$ regression fit itself
is in \cref{subsec:nanogpt-pn-trajectory}.

\paragraph{Gap-counting estimator $\hat\xi_\Lambda$.}
We analyse the four training-time GPT-124M runs (two PE datasets, two
RoPE) of \cref{subsec:setup-gpt124m} at the seven milestone
checkpoints. For each checkpoint, three evaluation seeds produce
independent validation batches ($B=16$ sequences, context lengths
$n\in\{64,128,256,512,768,1024\}$, block size $1024$). For each
(layer, head, $n$, batch) row we record the current-layer dot product
$z_{n,j}=\langle q,k_j\rangle/\sqrt d$ with $z^*_n:=\max_j z_{n,j}$.
Per-row $\Lambda_n$ and the contact coordinates
$(\alpha_n,\Delta_n)$ are computed from
\cref{def:contact} via the competitor gaps $z^*_n-z_{n,j}$ at
numerical tolerance $\epsilon=10^{-6}$. The body's
$\hat\xi_\Lambda$ is the bucket-mean slope estimator of
\cref{subsubsec:verification-protocol}: rows are pooled across
evaluation seeds into the index sets $\mathcal R(c,n)$, the bucket
mean is taken over those pooled rows at each of the six $n$ values,
and $\hat\xi_\Lambda$ is the OLS slope of those bucket means against
$\log\log n$. Because each of the three seeds contributes the same
$(\ell,h,\text{batch},n)$ structure, the pooled-row slope coincides
with the mean of three per-seed slopes; the body's headline
half-width is the bootstrap-IQR/2 over $(s,\ell,h)$ tuples of
\cref{subsubsec:verification-protocol}. Rows with
$\log_{10}\Lambda_n>5$ (exact-tie configurations) are
dropped before averaging, matching the Qwen protocol of
\cref{subsec:verification-protocol}.

\paragraph{Bootstrap intervals.}
For each cell $c$ let $\mathcal T(c)$ be the cell's bootstrap-unit
set: $\{(\ell,h)\}$ on Qwen and $\{(s,\ell,h)\}$ on nanoGPT, the
latter pooled across the three evaluation seeds. Each draw selects
$|\mathcal T(c)|$ tuples uniformly with replacement from
$\mathcal T(c)$, recomputes the bucket-mean slope $\hat\xi_X(c)$ from
the rows in $\mathcal R(c,n)$ that belong to the drawn tuples
(separately at each $n\in\mathcal N(c)$), and the half-IQR over
$B=200$ such draws is the half-width reported in
\cref{tab:verif-xi-summary}; tie\,\% follows the same convention. The
$\xi_\beta^*$ column carries no interval on Qwen ($0.5$-grid PPL
minimum) and the six-point OLS standard error of $\hat\xi_\beta$ on
nanoGPT.

\paragraph{$\beta_n$-fit estimator $\hat\xi_\beta$.}
At every checkpoint $\beta_n$ is read off the model state and fitted
to the bias-free power family $\beta_n=a_1(\log n)^{\xi_\beta}$ at the
same six $n$ that the gap-counting uses. The OLS slope
of $\log\beta_n$ on $\log\log n$ at those six points yields
$\hat\xi_\beta$ (closed-form least squares, no grid). The shared
domain matters: \cref{thm:deterministic-} predicts equality of
$\hat\xi_\Lambda$ and $\hat\xi_\beta$ in the limit, but only when both
estimators are read on the same $n$ set; this is the choice we make
throughout. The bias-free constraint $a_2\equiv 0$ is justified in
\cref{subsec:nanogpt-pn-trajectory}.

\paragraph{Limitations.}
The check is limited in two respects.
(i) \emph{Single training seed per configuration.} The
$\hat\xi_\Lambda$ band is over three evaluation batches at fixed
training seed; variation from initialisation and data ordering is not
measured. $\hat\xi_\beta$ has no eval-seed variance by construction.
(ii) \emph{Short context.} With $n\le 1024$ the log-log dynamic range
covers half a decade, which constrains the bucket-mean slope but
leaves higher-order curvature unconstrained; this is part of why
$\hat\xi_\Lambda$ and $\hat\xi_\beta$ retain a finite-$n$ gap.

\section{Acquisition protocol for the verification analysis on Qwen-7B-Chat}
\label{sec:verification-extras}

This appendix records the per-row attention acquisition pipeline
that backs the verification analysis of \cref{subsec:verification}
and the body figures it feeds.

\subsection{Acquisition protocol}
\label{subsec:verification-protocol}

We use Qwen-7B-Chat with dynamic-NTK RoPE enabled
(\cref{subsec:setup-qwen}). Attention rows are collected from all
$32$ layers and $32$ heads on PG19 and Proof-Pile-2 at context
lengths $\neval\in\{8192,16384,32768\}$, with four context-length
values per document at the ratios $n/\neval\in\{1/4,1/2,3/4,1\}$
and one query position per (layer, head, document, $n$) tuple. The
body's per-$\neval$ fit of \cref{subsubsec:verification-protocol} skips
$\neval=\ntrain=8192$ (training window, no extrapolation). The
upstream pass applies a length-dependent inverse-temperature
multiplier $\beta_n\propto(\log n)^{\xi_0}$ at the seven values
$\xi_0\in\{0,0.5,1,1.5,2,2.5,3\}$, where $\xi_0=0$ disables the
multiplier; $\xi_0$ perturbs the representation entering a layer
while the measured object is always the unrescaled current-layer
score $z_{n,j}=\langle\mathbf{q}_{n,i_n},\mathbf{k}_{n,j}\rangle/\sqrt{d_{qk}}$.
Across the three context lengths and four ratios, this yields eight
values $n\in\{2048,4096,6144,8192,12288,16384,24576,32768\}$.
The published-paper headline aggregation
(\cref{subsubsec:verification-protocol,tab:verif-xi-summary}) uses
the document subset $4$ from PG-19 and $5$ from Proof-Pile-2, so the
total attention-row count entering the headline statistics is
\(
   32\times 32 \times 3\,(\neval) \times 4\,(\text{ratios})
   \times 7\,(\xi_0) \times (4{+}5)\,(\text{docs})
   =774{,}144.
\)
The empirical contact gap is
the largest competitor gap $u$ satisfying
$\log N_n(u)/u\ge (1-\epsilon)\,\Lambda_n$ at $\epsilon=10^{-6}$,
a floating-point relaxation of the exact maximizer in
\cref{def:contact}.

\paragraph{Per-$\neval$ fit.}
The body $\hat\xi_{\rm pool}(\neval)$ is the OLS slope of the
count-weighted means of $\log\Lambda_n$ taken at each of the four
$n$ values within an $\neval$, regressed against
$\log\log n$. Sink-token attention rows with
$\log_{10}\Lambda_n>5$ are excluded from the bulk before
averaging. Aggregating across $\neval$ at shared $n$
mixes three dynamic-NTK RoPE rescale conditions (one per source
$\neval$) and biases the slope; per-$\neval$ fits avoid
this contamination.

\paragraph{Representative rows shown in body figures.}
The single attention rows recast as $N_n(t)$ in
\cref{fig:xi-coord-decomp} (middle, right) are selected
deterministically from the protocol above: for each contrast cell
$c$, take the row whose $(\log\alpha_n,\log\Delta_n)$ is closest to
the cell's bucket mean
$(\overline{\log\alpha}(c,n),\overline{\log\Delta}(c,n))$ in
$(\log\alpha_n,\log\Delta_n)$ at the largest $n\in\mathcal N(c)$. The
two cells used in body \cref{fig:xi-coord-decomp} are Qwen-7B-Chat
PG-19 at $\neval=32768$ ($n=32768$) and nanoGPT SlimPajama at
$\mathrm{iter}=100\mathrm{k}$ ($n=1024$); the contact statistics of
the selected rows are $(\hat\Lambda_n,\hat\alpha_n,\hat\Delta_n)
\approx(4.22,0.13,0.33)$ and $(1.64,0.26,1.09)$ respectively.


\end{document}